\documentclass{imsart}

\RequirePackage[OT1]{fontenc}
\RequirePackage{amsthm,amsmath}
\RequirePackage{natbib}
\RequirePackage[colorlinks,citecolor=blue,urlcolor=blue]{hyperref}

\startlocaldefs
\numberwithin{equation}{section}
\theoremstyle{plain}

\endlocaldefs

\usepackage{times}
\usepackage{graphicx} 
\usepackage{natbib}
\usepackage{color}

\usepackage{stmaryrd}

\renewcommand{\leq}{\leqslant}
\renewcommand{\geq}{\geqslant}

\newcommand{\inner}[2]{\langle #1, #2 \rangle}

\newcommand{\spanf}[1]{\langle #1 \rangle}

\usepackage{natbib}

\usepackage{amsthm}
\usepackage{amsmath}
\usepackage{algorithm}
\usepackage{algorithmic}
\usepackage{amssymb}
\usepackage{bbm}
\usepackage{mathtools}

\newtheorem{theorem}{Theorem}[section]
\newtheorem{corollary}{Corollary}[section]
\newtheorem{lemma}{Lemma}[section]
\newtheorem{innercustomthm}{Theorem}
\newenvironment{customthm}[1]
  {\renewcommand\theinnercustomthm{#1}\innercustomthm}
  {\endinnercustomthm}

\newtheorem{innercustomcor}{Corollary}
\newenvironment{customcor}[1]
  {\renewcommand\theinnercustomcor{#1}\innercustomcor}
  {\endinnercustomcor}

\usepackage{hyperref}

\DeclareMathOperator{\spn}{span}
\DeclareMathOperator{\Proj}{\mathbb{P}}
\DeclareMathOperator{\Categorical}{Categorical}
\DeclareMathOperator{\Crit}{Crit}
\DeclareMathOperator{\pen}{pen}

\newcommand{\MAP}{\text{MAP}}

\begin{document}

\begin{frontmatter}

\runtitle{Change point and clustering}
\title{Bayesian Model Selection for Change Point Detection and Clustering}

\begin{aug}
\author{\fnms{Othmane} \snm{Mazhar,}\thanksref{}\ead[label=e1]{othmane@kth.se}}
\author{\fnms{Cristian}\snm{ R. Rojas,}\thanksref{}\ead[label=e2]{cristian.rojas@ee.kth.se}}
\author{\fnms{Carlo} \snm{Fischione}\thanksref{}\ead[label=e3]{carlofi@kth.se
}}\\
\and
\author{\fnms{Mohammad} \snm{R. Hesamzadeh} \thanksref{} \ead[label=e4]{mrhesamzadeh@ee.kth.se}}

\runauthor{O. Mazhar Author et al.}

\affiliation{Royal Institute of technology\thanksmark{m1}}

\address{Othmane Mazhar and Cristian R. Rojas\\
Division of Decision And Control Systems\\
School of Electrical Engineering and Computer Science\\
KTH Royal Institute of Technology\\
SE-100 44 Stockholm, Sweden.\\
\printead{e1}
\phantom{E-mail:\ }\printead*{e2}}

\address{Carlo Fischione \\
Division of Network And Sytems Engineering\\
School of Electrical Engineering and Computer Science\\
KTH Royal Institute of Technology\\
SE-100 44 Stockholm, Sweden.\\
\printead{e3}}

\address{Mohammad R. Hesamzadeh \\
Division of Electric Power And Energy Systems\\
School of Electrical Engineering and Computer Science\\
KTH Royal Institute of Technology\\
SE-100 44 Stockholm, Sweden.\\
\printead{e4}}
\end{aug}

\begin{abstract} 
We address the new problem of estimating a piece-wise constant signal with the purpose of detecting its change points and the levels of clusters. Our approach is to model it as a nonparametric penalized least square model selection on a family of models indexed over the collection of partitions of the design points and propose a computationally efficient algorithm to approximately solve it. Statistically, minimizing such a penalized criterion yields an approximation to the maximum a-posteriori probability (MAP) estimator. The criterion is then analyzed and an oracle inequality is derived using a Gaussian concentration inequality. The oracle inequality is used to derive on one hand conditions for consistency and on the other hand an adaptive upper bound on the expected square risk of the estimator, which statistically motivates our approximation. Finally, we apply our algorithm to simulated data to experimentally validate the statistical guarantees and illustrate its behavior.
\end{abstract} 

\end{frontmatter}

\section{Introduction}

In many relevant scientific and engineering fields one is presented with time series data switching back and forth between different regimes. A classical estimation problem in this setting is the well-studied change point detection problem where one tries to estimate when some properties of the sequence of the random variables changes. This local property is of prime importance in many learning tasks such as signal segmentation \cite{Abou-Elailah16,Kim09}, change point detection in comparative genomics for early cancer diagnosis \cite{Lai05}, and modeling and forecasting of changes in financial data \cite{Lavielle06,Spokoiny09}.

For other applications, one needs more than this local answer and is interested in a more general overview of the time series where for instance earlier data samples behave like new ones. Examples of this are found in: electricity market data, where prices might have different behavior corresponding to different price regimes that might reappear depending on some triggering events; signal partitioning with some parts of the signal sharing similar properties; and speech segmentation with different alternating sources. Generally speaking, it is of interest in these situations to determine the change points and the clusters for a more precise description of the inhomogeneous time series. 

Parametric models for solving the change point detection problem have been proposed in \citet{Cleynen14} and \citet{Rigaill12}. However, in dealing with the change point and clustering problem we would naturally require that our solution does not assume any knowledge of the number of changes nor the actual number of clusters, as these numbers would evolve over time, so we expect new changes in the process to happen and new clusters to form as $N$, the number of samples, grows. Thus, any practical procedure should be able to estimate these numbers and also have adaptive guarantees with respect to how fast these numbers grow. Similar setups for change point detection have been the subject of study by \citet{Harchaoui07a}, \citet{Arlot16} and \citet{Garreau17} who use characteristic kernels for detecting changes in the distribution, while from a computational stand point a more effective implementation has been proposed by \citet{Celisse17}. In this study, we will restrict ourselves to an \emph{iid} (independent and identically distributed) Gaussian sequence model of the data with known variance, noting that the same study can be done using kernels and that the algorithm we develop can be effectively implemented using the same procedure as in \cite{Celisse17}, as explained later in the paper.

Two other related lines of research, but which we do not explore here, are on-line algorithms for segmentation and $L_1$-regularized segmentation. We refer the reader to \cite{Tartakovsky14} for an extensive review of on-line algorithm for change point detection. Data segmentation using the $L_1$-penalty, also known as total variation denoising, was introduced by \citet{Rudin92}. The one dimensional case, corresponding to the Fussed LASSO, has been studied in \citep{Tibshirani05} and \citep{Rennie69} and an efficient algorithm  has been proposed by \citet{Arnold16}. More recent results can be found in \citep{Dalalyan17} for the one dimensional case and \citep{Hütter16} for two dimensional case.

\textbf{Main contribution}: The change point detection and clustering for sequences of data points does not seem to have been previously studied. In this work we propose a two-pass dynamic programming algorithm for selecting an adequate model from a collection of candidate models. We motivate the choice of the algorithm computationally by showing that it runs in $\mathcal{O}(N^2D+D^4)$ time (where $D$ is an upper bound on the number of change points), statistically by showing that it can be seen as an approximation of a computationally hard MAP optimization problem for which we can derive an oracle inequality that guarantees low sample complexity, consistency and adaptivity, and practically by testing the model on simulation data.

\textbf{Structure of the paper}: In Section~2 we formulate the problem as one of nonparametric model selection from a family of models over all partitions of the data set. After some preliminaries and notations are given in Section~3, we propose in Section~4 a two-pass dynamic programing algorithm as a computationally effective relaxation of the optimization criterion and analyze its computational cost. We then put the model selection problem in a  Bayesian framework in Section~5, and use a Laplace-type approximation to derive as optimization criterion the maximum a-posteriori probability. In Section~6 we derive an oracle inequality for the criterion that our algorithm is approximating, and study its properties. Experimental results showing that the clusters and segments can be obtained effectively estimated are presented in Section~7 using simulation data.

\section{Problem formulation}

Let $\mathcal{Y}$ be a measurable space and $Y_1,Y_2,\dots,Y_N \in \mathcal{Y}$ denote random variables with distributions $P_{Y_i}$. Our goal is on one hand to detect changes in the sequence of distribution measures $(P_{Y_i})_{i=1}^N$ and on the other hand to cluster the data points coming from the same process. Hence we put random variables between two consecutive changes in the same segment, and we think of random variables of the same segment or different segments as belonging to the same cluster, if they are the realization of the same process. 

One important case both in theory and in practice is the uniform constant design model were the $Y_i$s depend on deterministic variables uniformly spaced on a grid $X_i=i$ for $i \in \llbracket 1,N \rrbracket := \{1, \dots, N\}$ through a regression function $f^*$ with an additive \emph{iid} random noise $(\epsilon_i)_{i=1}^N$. Taking the distribution of the $\epsilon_i$'s as $\mathcal{N}(0,\sigma^2)$ with known variance, we end up with the following Gaussian sequence model:
\begin{equation} \label{regression}
Y_i=f_i^*+\epsilon_i, \quad \text{ for } \ i \in \llbracket 1,N \rrbracket.
\end{equation}
Here we are placed in a regression setting of the form $Y=f^*+\epsilon$, where $Y = [Y_1\,\,\cdots\,\,Y_N]^T$, $f^* = [f_1^*\,\,\cdots\,\,f_N^*]^T$ and $\epsilon = [\epsilon_1\,\,\cdots\,\,\epsilon_N]^T \sim \mathcal{N}(0,\sigma^2I_N)$, and we are interested in estimating $f^*$ as a piecewise constant function that takes a limited number of values. 

We emphasize that it is unlikely that the data correspond exactly to a piecewise constant function plus independent random Gaussian noise and that we are in this low dimensional hidden structure exactly, yet there might exist a good sparse linear approximation. Hence our search is not for an exact model, rather we are trying to select the best model in a collection of candidates, as we explain in the next section.

\section{Preliminaries and notation}

We would like to perform dimensionality reduction by exploiting the hidden structure on the data sequence $Y_1,Y_2,\dots,Y_N$. To do this we split it into different segments while also putting the segments sharing same mean into the same cluster. Hence if we knew the clusters our problem reduces to fitting a constant to a set of observations over each cluster. Observe that if $f^*$ is constant over parts of $\llbracket 1, N \rrbracket$, the it determines a clustering of $Y_1,Y_2,\dots,Y_N$ over the values where it is constant. Hence, we can think about the problem as, first determining the clustering of the $Y_1, Y_2, \dots, Y_N$ which would result in a partition $\pi$ of $\llbracket 1,N \rrbracket$, and then choosing the best value of $\hat{f}$ over each part as our estimate. So $f^*$ belong to the subspace $\mathcal{F}_{\pi}$: subspace of functions that are constant over the parts of the partition $\pi$. 

To formalize this, let $\mathcal{M}$ be an index set over the collection of partitions $\Pi_N$ of $\llbracket 1,N \rrbracket$; given $m \in \mathcal{M}$, denote by $\mathcal{F}_m$ the subspace of functions that are constant over the parts of $\pi_m$. Our goal is two-fold: find $\hat{m}$ as the index estimate of $\mathcal{F}_{\hat{m}}$, the subspace where the estimate of $f^*$ lives, and from $\mathcal{F}_{\hat{m}}$ compute $\hat{f}_{\hat{m}}$ as our estimate. We represent a partition $\pi$ as an unordered collection of its subsets $\pi=\{[1],[2],\dots,[|\pi|]\}$ with $[k]$ being the $k^{th}$-equivalent class, -part or -cluster, and $|\pi|$ the cardinality of the partition. Every part $[k]$ can be seen as the union of segments $[k]=\{[k_1],[k_2],\dots,[k_{|[k]|}]\}$ where $(k_i)_{i=1}^{|[k]|}$ is the collection of maximal intervals in $[k]$ that we call segments of the $k^{th}$-cluster. The last element in each segment $[k_i]$ is called a change point. We define $d'_m := |\pi_m|-1 = \dim(\mathcal{F}_m)-1$ as the clustering dimension. Even though this choice might create some confusion it will be consistent the notations used in the proofs of sections 5 and 6. Also we define $d''_m= | \pi_m|_0 := |\cup_{k=1}^{d_m'+1}[k]|$ as the change point dimension. 

To link partitions to subspaces let $e_l := (0,\dots,1,\dots,0)$ be the $l^{th}$-component of the standard orthonormal basis of $\mathbb{R}^N$, and define for a subset $A$ of $\llbracket 1,N \rrbracket$ the vector $\mathbbm{1}_A:=\sum_{l \in A} e_l$. For $[k]$, the $k^{th}$ cluster of $\pi_m$, with a slight abuse of notation we define $\mathbbm{1}_{[k]} := \sum_{i=1}^{|[k]|} e_{k_i}$, and observe that $\mathcal{F}_m=\spn\{\mathbbm{1}_{[k]}\colon k \in \pi_m\}$, which is consistent with the definition of the clustering dimension $d'_m := |\pi_m|-1 = \dim(\mathcal{F}_m)-1$.

We define $\spanf{f^*} := \spn\{f^*\}$, $\mathcal{S}_1 \oplus \mathcal{S}_2$ as the direct sum of the two vector space $\mathcal{S}_1 $ and $ \mathcal{S}_2$, and $\mathcal{S}_1 \ominus \mathcal{S}_2$ as their direct difference. $\Proj_{\mathcal{S}}$ denotes the (orthogonal) projection operator onto the subspace $\mathcal{S}$. We also define the partitions inclusion as $m_1 \subset m_2$ if $\mathcal{F}_{m_1} \subset \mathcal{F}_{m_2}$, or equivalently if $\pi_{m_2}$ is finer than $\pi_{m_1}$. 

\begin{figure}[t]
\begin{center}
\centerline{\includegraphics[width=\columnwidth]{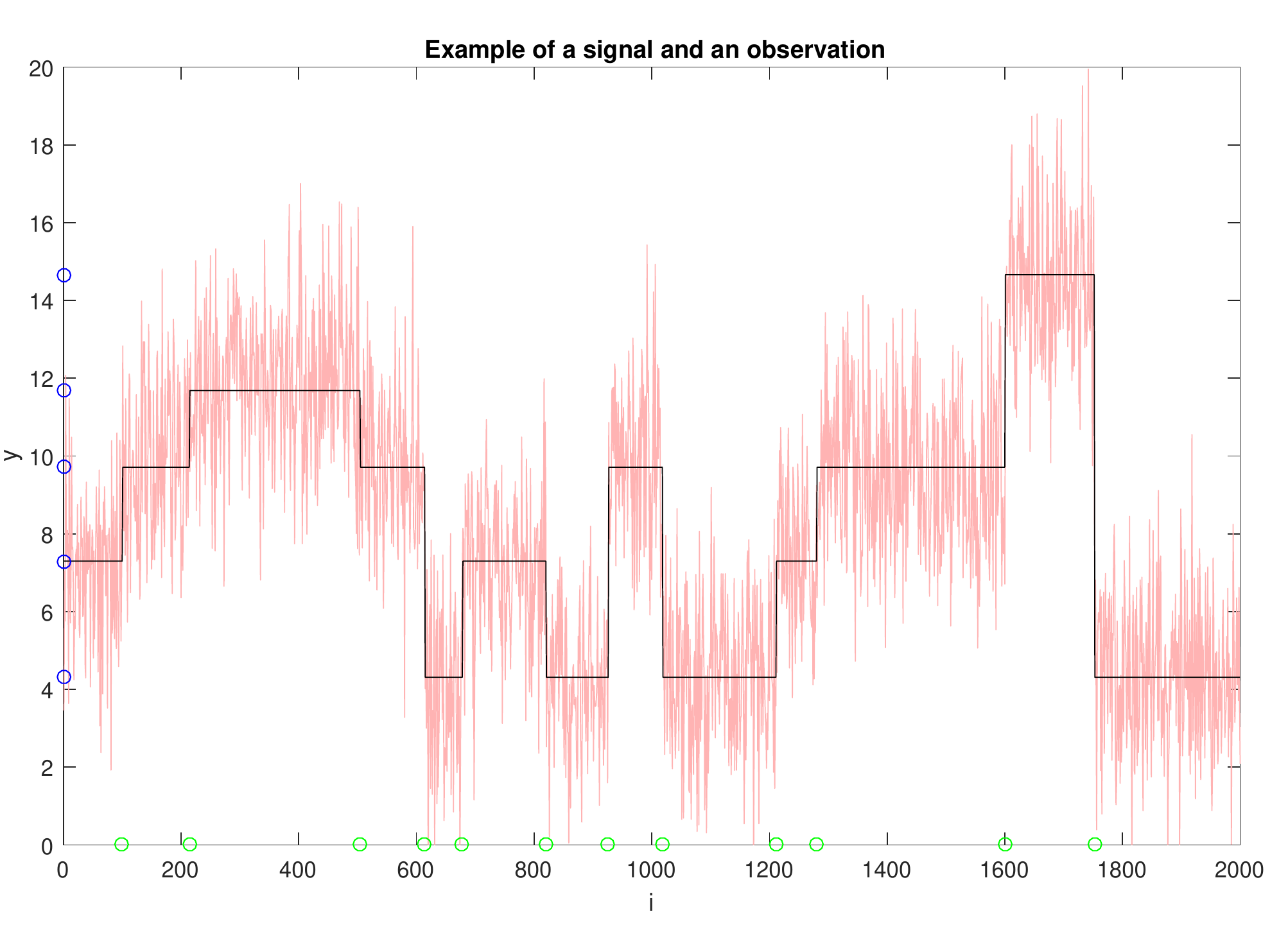}}
\caption{Example of a piecewise constant signal $f^*$ (black line) and observed signal $Y$ (pink line) with clustering values  (blue circles) and change points (green circles).}
\label{fig1}
\end{center}
\end{figure}

\textbf{Example 1.} \begin{small} \it
Consider the signal $f^*$ of Figure~\ref{fig1}, whose partition is 
\begin{align} \label{example}
\pi = \{ [ 1 ];[ 2 ];[ 3 ];[ 4 ];[ 5 ] \},
\end{align}
where
\begin{align*}
[1] &=\llbracket 615,678 \rrbracket \cup \llbracket 821,926 \rrbracket \cup \llbracket 1019,1211 \rrbracket\ \cup \llbracket 1753,2000 \rrbracket \\
[2] &=\llbracket 1,100 \rrbracket \cup \llbracket 679,820 \rrbracket \cup \llbracket 1212,1280 \rrbracket\ \\
[3] &=\llbracket 101,214 \rrbracket\ \cup \llbracket 505,614 \rrbracket \cup \llbracket 926,1018 \rrbracket \cup \llbracket 1281,1600 \rrbracket \\
[4] &=\llbracket 215,504 \rrbracket\\
[5] &=\llbracket 1601,1752 \rrbracket.
\end{align*}
Hence, $d_{\pi}'= 4$ and $d_{\pi}''= 12$ for this signal.
\end{small}
 
We also denote by $C_k^N$ the binomial coefficient that gives the number of ways, disregarding order, that $k$ objects can be chosen from among $N$ objects. This is given by
\begin{align}
C_k^N := \frac{N!}{k!(N-k)!}. \label{binomial coefficient}
\end{align}
The Stirling numbers of the second kind, $S(N,k)$, correspond to the number of ways to partition a set of $N$ objects into $k$ non-empty subsets, or, similarly, to the number of different equivalence relations with precisely $k$ equivalence classes that can be defined on an set of $N$ elements.

We are precisely interested in the case where the element set is $\llbracket 1,N \rrbracket$ and the distance between every two elements in each equivalence class is at least 2; we denote the number of such equivalent classes by $S^2(N,k)$. $S(N,k)$ and $S^2(N,k)$ satisfy the following recurrence relations:
\begin{align}
S(N,k) &= S(N-1,k-1) + kS(N-1,k), \quad N \geq k, \nonumber \\
S^2(N,k) &= S(N-1,k-1), \quad N, k \geq 2. \label{recursive Stirling}
\end{align}
For the proofs of these results, we refer the reader to \cite{ConcreteMathematics} and \cite{AMohr}.

\section{Two-pass dynamic programming for change point detection and clustering}
To solve the change point and clustering problem, a natural approach is to consider the minimization of a criterion of the form
\begin{align} \label{Criterion}
\Crit(m) = \|y-\hat{f}_m\|_2^2 + \sigma^2K\pen(m),
\end{align}
with a penalty term $\pen(m)$ depending only on $d'$ and $d''$ and a multiplicative tuning parameter $K$. Indeed, as we shall see the penalty can be chosen such that the minimizer $\hat{f}_{\hat{m}}$ of \eqref{Criterion} behaves like an approximation to a maximum a-posteriori estimator (MAP), and also, the average expected risk $\frac{1}{N}\mathbb{E}[\|\hat{f}_{\hat{m}}-f^*\|_2^2] \to 0$ for a large class of signals $f^*$, namely, those corresponding to models with $d'\leq d''= o(N/\ln N)$, \emph{i.e.}, $f^*$ is a consistent estimator for those signals. The specific form of $\pen(m)$ will be derived in the next section, based on an oracle inequality that will guarantee consistency and adaptivity of our estimator.

Although the estimator $\hat{f}_{\hat{m}}$ enjoys good statistical properties, from a computational stand it would involve the exploration of $\mathcal{M}$. The set $\mathcal{M}$ is identified with the collection of all the partitions of $\llbracket 1,N \rrbracket$, whose number asymptotically behaves like $\mathcal{O}(Ne^N/\ln N )$, rendering the minimization of the criterion~\eqref{Criterion} computationally challenging. A way to bypass this issue for the change point only detection problem is via dynamic programming~\cite{Harchaoui07a}; this approach works in this simplified setup since there is a natural ordering for exploring the subproblems, which does not hold here. To overcome this, we will relax the criterion in such a way to create a subproblem ordering and thus derive a computationally feasible approximation. The proposed new method is outlined in Algorithm~\ref{two passes Dynamic programming algorithm}.

\begin{algorithm}
\begin{small}
   \caption{Two-Pass Dynamic Programming Algorithm}
   \label{two passes Dynamic programming algorithm}
\begin{algorithmic}[1]
   \INPUT data points $(y_i)_{i=1}^N$, maximum number of changes $D$ and penalty strength $K$.
   \STATE {\begin{align*}
   \\[-1.2cm] \bar{y}_{[k,l]} &:= \frac{\sum_{i=k}^{l} Y_i}{k-l+1} \\
   R_{[k,l]} &:= \sum_{i=k}^{l} (y_i-\bar{y}_{[k,l]})^2, \quad 1 \leq k \leq l \leq N.
   \end{align*}}
%
   \FOR {$d=1$ {\bfseries to} $D$}
   \STATE \label{step:1} use the dynamic programing recurrence in  \eqref{dynamic programming 1} and a backtracking step to compute 
   \begin{align}
   C_d(N) &:= \min \limits_{|\bar{m}|=d} \|Y - \Proj_{\mathcal{F}_{\bar{m}}} Y\|^2, \label{change point detection} \\
   \tilde{m}_d &\in \arg \min \limits_{|\bar{m}|=d} \|Y - \Proj_{\mathcal{F}_{\bar{m}}} Y\|^2. \nonumber
   \end{align}
    \ENDFOR
   \FOR{$d=1$ {\bfseries to} $D$}
   \STATE \begin{align*}
   \\[-1.2cm] \tilde{m}_d &=: \{0 \leq i_1 < i_2 < \dots < i_d < N\} \\
   (\alpha_k)_{k=0}^{d} &:= (i_{k+1}-i_k)_0^{d}, \quad (i_0=0, i_{d+1}=N).
   \end{align*}
   \STATE sort $(\bar{y}_{[1,i_1]}, \bar{y}_{[i_1+1,i_2]}, \dots, \bar{y}_{[i_d+1,N]})$.\label{step:2}
   \STATE $\left(\bar{y}_{(k)}\right)_{k=0}^d := $ ordered sequence of $(\bar{y}_{[i_k+1,i_{k+1}]})_{k=0}^{d}$ $(\alpha_{(k)})_{k=0}^{d} := $ corresponding permuted $(\alpha_k)_{k=0}^{d}$ according to permutation $\phi_d$.
   \STATE $\bar{\bar{y}}_{(k,l)} := \frac{\sum_{i=k}^{l} \alpha_{(i)}\bar{y}_{(i)}}{\sum_{i=k}^{l-1} \alpha_{(i)}}$ and $\bar{R}_{[k,l]} := \sum_{i=k}^{l} \alpha_{(i)}(\bar{y}_{(i)}-\bar{\bar{y}}_{(k,l)})^2$, $1 \leq k \leq l \leq d$.
\FOR{$\delta=1$ {\bfseries to} $d$}
   \STATE \label{step:3} use the dynamic programing recurrence in  \eqref{dynamic programming 2} and a backtracking step to compute  
   \begin{align} \label{clustering}
   G_{(d,\delta)} &:= \min \limits_{m \in \mathcal{M}_{\bar{y}_{\tilde{m}},\delta}} \| \Proj_{\mathcal{F}_{\bar{m}}} Y - \Proj_{\mathcal{F}_{m}}\Proj_{\mathcal{F}_{\bar{m}}}  Y\|^2, \\
   \tilde{\tilde{m}}_{(d,\delta)} &\in \arg \min \limits_{m \in \mathcal{M}_{\bar{y}_{\tilde{m}}\delta}} \| \Proj_{\mathcal{F}_{\bar{m}}} Y - \Proj_{\mathcal{F}_{m}}\Proj_{\mathcal{F}_{\bar{m}}} Y\|^2. \nonumber
   \end{align} 
     \ENDFOR
   \ENDFOR
   \STATE $B_{(d,\delta)} := C_d + G_{(d,\delta)} + \sigma^2K \pen((d,\delta))$, $1 \leq \delta \leq d \leq D$.
   \STATE \label{step:4} $(\hat{d},\hat{\delta}) := \arg \min \limits_{1 \leq \delta \leq d \leq D }B_{(d,\delta)}$.
   \STATE \label{step 5} reconstruct $m_{(\hat{d},\hat{\delta})}$ from $\tilde{m}_{\hat{d}}$ and $\tilde{\tilde{m}}_{(\hat{d},\hat{\delta})}$.
   \OUTPUT value of criterion $\Crit(m_{(\hat{d},\hat{\delta})}) =B_{(\hat{d},\hat{\delta})}$ and selected model for change points and clusters $m_{(\hat{d},\hat{\delta})}$.
\end{algorithmic}
\end{small}
\end{algorithm}

Let $\bar{y}_{[k]} := (\sum_{i \in [k]} Y_i)/|[k]|$, the average of the elements of $Y$ in the $[k]$-th part. Notice that, given $\pi_m = \{ [ 1 ];[ 2 ];\dots;[ d'_m -1] \}$, we have
\begin{align*}
\Proj_{\mathcal{F}_m}Y = \sum \limits_{k=1}^{d'_m-1}  \frac{\inner{Y}{\mathbbm{1}_{[k]}}}{\|\mathbbm{1}_{[k]}\|^2}\mathbbm{1}_{[k]} 
=\sum \limits_{k=1}^{d'_m-1} \bar{y}_{[k]} \mathbbm{1}_{[k]}. 
\end{align*}

The minimization of criterion~\eqref{Criterion} can then be equivalently written as 
\begin{equation*}
\begin{aligned}
\min \limits_{m \in \mathcal{M}} \Crit(m) 
&= \min \limits_{m \in \mathcal{M}} \{ \|y-\Proj_{\mathcal{F}_m}Y\|_2^2 + \sigma^2K\pen(d'_m,d''_m) \} \\
&= \min \limits_{0 \leq d' \leq d'' \leq D} \left\{ \min   \limits_{\substack{|m|=d'\\ |m|_0=d''}} \|y-\Proj_{\mathcal{F}_m}Y \|_2^2 + \sigma^2K\pen(d',d'') \right\},
\end{aligned}
\end{equation*}
where $D$ is a reasonable upper bound on the number of change points. As we shall see later, from a statistical point of view there is no need to explore all possible values of $d'$ and $d''$, since the statistical guarantees only hold in a regime where $d' \leq d'' = o(N / \ln N)$.

We define $\pi_{\bar{m}}$ to be the partition having as elements all the segments of $\pi_{m}$ and instead of computing the minimum exactly we will take a greedy step by defining $$\tilde{m} := \arg \min_{|\bar{m}|=d''} \|Y - \Proj_{\mathcal{F}_{\bar{m}}} Y\|^2$$ and defining $\mathcal{M}_{\tilde{m},d'}: = \{ m \in \mathcal{M}\colon m \subset \tilde{m},|m|=d'\} $, which can be identified with the collection of all partitions of $\llbracket 1,d'' \rrbracket$ into $d'$ sets. We restrict further this collection to partitions $\pi$ satisfying what we call the \textbf{clustering property}, which states that if $\mathbb{I}_1$, $\mathbb{I}_2$ and $\mathbb{I}$ are segments in some (possibly different) parts of $\pi$, then 

\begin{align}\label{clutering property}
    \begin{cases*}
     \mathbb{I}_1, \mathbb{I}_2 \in [k] \\
     \bar{y}_{\mathbb{I}_1} \leq \bar{y}_{\mathbb{I}} \leq \bar{y}_{\mathbb{I}_2}
    \end{cases*}
    \quad \Rightarrow \quad \mathbb{I} \in [k].
\end{align}
This sub-collection will be denoted as $\mathcal{M}_{\bar{y}_{\tilde{m}},d'}$. Simply put, this property says that the partitions considered are those that respect the ordering of $(\bar{y}_{[i_k+1,i_{k+1}]})_{k=0}^{d''}$, since if two segments 
$\mathbb{I}_1,\mathbb{I}_2$  belong to $[k]$, and the segment $\mathbb{I}$ satisfies $\bar{y}_{\mathbb{I}_1} \leq \bar{y}_{\mathbb{I}} \leq \bar{y}_{\mathbb{I}_2}$, then it should also be in cluster $[k]$.

This leads to the following upper bound, whose detailed derivation is given in appendix B:
\begin{align*}
\min \limits_{m \in \mathcal{M}} \Crit(m)  &\leq \min \limits_{0 \leq d'' \leq D} \bigg\{ \min \limits_{|m|=d''} \|Y - \Proj_{\mathcal{F}_{m}} Y\|^2 \\ 
&\qquad + \min \limits_{\substack{0 \leq d' \leq d'' \\ m \in \mathcal{M}_{\bar{y}_{\tilde{m}},d''}}} \| \Proj_{\mathcal{F}_{\tilde{m}}} Y - \Proj_{\mathcal{F}_{m}}\Proj_{\mathcal{F}_{\tilde{m}}} Y\|^2   + \sigma^2K\pen(d',d'') \bigg\}.
\end{align*}

Therefore, we can define the following relaxation for the minimization of the criterion in ~\eqref{Criterion}:
\begin{align} \label{equation for the algorithm}
\Crit_r(d'') &:=   \min \limits_{|m|=d''} \|Y - \Proj_{\mathcal{F}_{m}} Y\|^2 \nonumber \\ 
&+ \min \limits_{\substack{0 \leq d' \leq d''\\ m \in \mathcal{M}_{\bar{y}_{\tilde{m}},d''}}} \Bigg\{ \| \Proj_{\mathcal{F}_{\tilde{m}}} Y - \Proj_{\mathcal{F}_{m}}\Proj_{\mathcal{F}_{\tilde{m}}} Y\|^2  + \sigma^2K\pen(d'_m,d''_m) \bigg\}.
\end{align}
and our algorithm computes $\min \limits_{0 \leq d'' \leq D} \Crit_r(d'')$ and returns $m_{(\hat{d},\hat{\delta})}$. From this last definition we observe that
\begin{align*}
\min \limits_{m \in \mathcal{M}} \Crit(m) \leq \Crit(m_{(\hat{d},\hat{\delta})})=\min \limits_{0 \leq d'' \leq D} \Crit_r(d'').
\end{align*}
Thus, obtaining $m_{(\hat{d},\hat{\delta})}$ ensures making progress toward the minimization of $\Crit(m)$. The Two-Pass Dynamic Programming Algorithm~\ref{two passes Dynamic programming algorithm} is aimed at doing this by computing the value of the minimum in ~\eqref{equation for the algorithm} and returning a solution $\hat{m}=m_{(\hat{d},\hat{\delta})}$ in the following way:

\medskip
\textbf{Details of Main Steps in Algorithm~\ref{two passes Dynamic programming algorithm}}
\begin{small}
\begin{itemize}
\item \textbf{Step~3:} It computes $C_d(n)$ defined in ~\eqref{dynamic programming 1} for all $d$ and $n$ to obtain $C_d(N)$ for all $d \in \llbracket 1,N \rrbracket$. It does so by using a dynamic programming algorithm that computes recursively for all $ 2 \leq d \leq D$ and $d \leq n \leq N$ the following recurrence, similar to the one in \citet{Jackson05}:
\begin{align} \label{dynamic programming 1}
C_1(n) &:= R_{[1,n]} \\
C_d(n) &:= \min \limits_{i \in \llbracket d,n \rrbracket} \{ C_{d-1}(i-1) + R_{[i,n]}\}, \; d \geq 2. \nonumber
\end{align}
\item \textbf{Step~7:} For all values of $d$, it sorts the obtained segments according to their levels to yield $\left(\bar{y}_{(k)}\right)_0^d$, and it keeps track of the segments' sizes as $(\alpha_k)_{k=0}^{d}=(i_{k+1}-i_k)_0^{d}$. 
\item \textbf{Step~11:} It runs a modified dynamic programming recurrence on $\left(\bar{y}_{(k)}\right)_0^d$ that uses weights according to the sizes $(\alpha_{(k)})_{0}^{d}$. It does so using the following recurrence for all $1 \leq \delta \leq t \leq d$:
\begin{align}\label{dynamic programming 2}
G_{(t,1)} &:=  \bar{R}_{[1,t]}, \\
G_{(t,\delta)} &:= \min \limits_{i \in \llbracket \delta,t \rrbracket} \{ G_{(i-1,\delta-1)} + \bar{R}_{[i,t]}\}, \; \delta \geq 2. \nonumber
\end{align}
\item \textbf{Step~15:} It computes the minimum in ~\eqref{equation for the algorithm} and finds for which model it is attained by solving the minimization problem:
\begin{align*}
(\hat{d},\hat{\delta}) := \arg \min \limits_{1 \leq \delta \leq d \leq D }B_{(d,\delta)}.
\end{align*}
\item \textbf{Step~16:} It finally reconstructs $m_{(\hat{d},\hat{\delta})}$ from $\tilde{m}_{\hat{d}}$ and $\tilde{\tilde{m}}_{(\hat{d},\hat{\delta})}$ using the permutation $\phi(\hat{d})$.
\end{itemize}  
\end{small}

This algorithm can be thought of as an efficient way to compute the relaxation in ~\eqref{equation for the algorithm}, based on solving the change point detection problem in ~\eqref{change point detection} using the dynamic programing recurrence of ~\eqref{dynamic programming 1}, followed by a solving a clustering problem in ~\eqref{clustering} using the dynamic programing recurrence of ~\eqref{dynamic programming 2}.

The next theorem shows that Algorithm~\ref{two passes Dynamic programming algorithm} correctly solves the minimization problem in ~\eqref{equation for the algorithm} and explicits its time and space complexity.
\begin{theorem}\label{computaional complexity}
Let $(y_i)_{i=1}^N \subset \mathbb{R}$, $D \in \mathbb{N}$ and $K>0$. Then,
\begin{itemize}
\item for all $1 \leq d \leq D$,
\begin{align*}
\tilde{m}_d \in \arg \min \limits_{|\bar{m}|=d} \|Y - \Proj_{\mathcal{F}_{\bar{m}}} Y\|^2,
\end{align*}
\item for all $1 \leq \delta \leq d \leq D$,
\begin{align*} 
\tilde{\tilde{m}}_{(d,\delta)} \in \arg \min \limits_{m \in \mathcal{M}_{\bar{y}_{\tilde{m}}\delta}} \| \Proj_{\mathcal{F}_{\bar{m}}} Y - \Proj_{\mathcal{F}_{m}}\Proj_{\mathcal{F}_{\bar{m}}}  Y\|^2.
\end{align*}  
\end{itemize}
Furthermore, Algorithm~\ref{two passes Dynamic programming algorithm} correctly solves the minimization problem in ~\eqref{equation for the algorithm}, with time and space complexity $\mathcal{O}(N^3+D^4)$ and $\mathcal{O}(N^2+D^3)$, respectively. 
\end{theorem}
\begin{proof}
See Appendix~B.
\end{proof}

The time and space complexity can be improved to $\mathcal{O}(N^2D+D^4)$ and $\mathcal{O}(DN+D^3)$, respectively. We refer the reader to the discussion after the proof in Appendix~B for the derivation of this result. In this way we obtain a computationally feasible algorithm that finds the minimum in \eqref{equation for the algorithm} and returns an approximation to the criterion in \eqref{Criterion}. In the next section we will motivate the use of Algorithm~\ref{two passes Dynamic programming algorithm} from a statistical point of view by showing that the minimization of criterion~\eqref{Criterion} can be viewed as an approximate maximum a-posteriori estimator. 

\section{Model selection criterion for change point detection and clustering}

In this part we provide a derivation of the optimization criterion in ~\eqref{Criterion}. We start by proposing a Bayesian model selection scheme, which is later inverted to arrive at an integral form of the maximum a-posteriori probability (MAP) estimator. Then we use a Laplace approximation to derive turn the MAP into an optimization problem of the desired form.

Here we show that the proposed selection criterion in ~\eqref{Criterion} follows naturally from a Bayesian reasoning. For this, we model the data as being the outcome of the following sampling model. The observation $Y$ is generated from a multivariate Gaussian of mean $F$ and variance $\sigma^2I_N$ as described by ~\eqref{regression}. For the random variable $F$, given that it belongs to a subspace $\mathcal{F}_m$, we choose an absolutely continuous measure $\mathcal{L}^{d'_m}$ with respect to $\mathcal{\lambda}^{d'_m}$, the Lebesgue measure on $\mathbb{R}^{d'_m+1}$, such that $d\mathcal{L}^{d'_m}=  l_{f/m}d\mathcal{\lambda}^{d'_m+1} = \prod \limits_{k=1}^{d'_m+1}(l_{f_k/m}d\mathcal{\lambda})$ with $l_{f_1/m} = \dots =l_{f_{d'_m+1}/m}$. Later we will see that the choice $l_{f/m}$ will not matter in comparison to the order of approximation, nevertheless we would like it to be a bounded continuous prior satisfying some additional conditions given in Lemma~\ref{approx lemma}, even though we might be chosen as an improper prior. On the family of models $\mathcal{M}$ we impose a categorical distribution measure $\mathbb{P}_\mathcal{M}$ as prior, with a weight $p_m$ for model $m$. Thus, we obtain the following sampling model for the data\footnote{Here and in the sequel, the dependence of $p_m$ and $\mathbb{P}_\mathcal{M}$ on the number of samples $N$ is omitted, for simplicity of notation.}:
\begin{align}
Y/F &\sim \mathcal{N}(F,\sigma^2I_N) \nonumber \\ 
F/m &\sim \mathcal{L}^{d'_m} \label{sampling scheme} \\  
m &\sim \mathbb{P}_\mathcal{M}=\Categorical((p_m)_{m\in \mathcal{M}}). \nonumber
\end{align}
Since $Y$, $F$ and $m$ are now random variables, it makes sense to compute $\mu_{m/Y}$, the posterior distribution of $m$ given $Y$, and maximize it, to arrive at a MAP estimate of $m$ given bellow. 
{\small
\begin{align} \label{apposteriori}
p_{m/Y}=\frac{p_m \displaystyle \int_{f \in \mathcal{F}_m} \phi_N \left(\frac{Y-f}{\sigma}\right) l_{f/m}(f)df}{\sum_{m' \in \mathcal{M}} p_{m'} \displaystyle \int_{f' \in \mathcal{F}_m} \phi_N \left(\frac{Y-f'}{\sigma}\right) l_{f/m'}(f')df'}. 
\end{align}}
For the complete derivation of the formula in  \ref{apposteriori} we refer you to appendix B.

Starting from the a-posteriori distribution~\eqref{apposteriori} we can derive an approximation for the MAP as follows:
\begin{align}
p_{m/Y} &\propto p_m \int_{f \in \mathcal{F}_m} \phi_N \left(\frac{Y-f}{\sigma}\right) l_{f/m}(f)df \nonumber \\
&= p_m \prod_{k=1}^{d_m'+1} \frac{1}{(2\pi\sigma^2)^{\frac{|[k]|}{2}}}   \cdot \int_{\mathbb{R}} \exp\left(-\frac{\|y_{[k]}-f_k\mathbbm{1}_{[k]}\|_2^2}{2\sigma^2}\right) l_{f_k/m}(f_k)df_k. \label{for approx}
\end{align}

In the last step of \eqref{for approx} we define $y_{[k]}$ as the vector obtained from the entries of $y$ corresponding to cluster $[k]$. To obtain an approximation of the MAP estimate as a solution of a criterion of the form~\eqref{Criterion} we need  the result of lemma \ref{approx lemma} stated and proved in Appendix C using a Laplace approximation type of argument. We then obtain the following upper bound for the MAP for all $K \geq 1$ :
\begin{multline} \label{upper bound for the MAP}
\Crit_\MAP(m) \leq \frac{\|y-\Proj_{\mathcal{F}_m}y\|_2^2}{2\sigma^2} + K\left(\ln \frac{1}{p_{m}} +\frac{1}{2}(d_m'+1)\ln\frac{N}{d_m'}\right)  +\mathcal{O}(d_m').
\end{multline}
The complete derivation of \eqref{upper bound for the MAP} can be found in Appendix~C. Now we define our approximate MAP criterion as:
\begin{align}\label{Crit equation}
\Crit(m) &= \|y-\Proj_{\mathcal{F}_m}y\|_2^2 +\sigma^2K \pen(m) , \nonumber \\
\pen(m) &= \left(2 \ln \frac{1}{p_{m}} +(d_m'+1)\ln\frac{N}{d_m'}\right).
\end{align}
In the next section we finish the specification of the penalty term by providing the probabilities $p_m$ over the space of models. To do so we will exhibit an oracle inequality satisfied by the estimator that minimizes \eqref{Criterion}, and choose a probability mass function $(p_m)$ that gives a reasonable upper bound on the expected quadratic risk defined below.

\section{Oracle inequality and upper bound for the risk}
The standard way of assessing the performance of a statistical algorithm is by comparing its performance to a reasonable oracle. For this we use as a measure of performance of an estimator $\hat{f}$ the expected quadratic risk:
\begin{align*}
\mathcal{R}_n(\hat{f}) = \mathbb{E}[\|\hat{f}-f^*\|_2^2].
\end{align*}
In the case of the change point detection and clustering problem the comparison should be non-asymptotic, reflecting our lack of knowledge about both the clustering dimension and the change point dimension. For this we state below a non-asymptotic oracle inequality for $\Crit(m)$ using an oracle with remainder of the form:
\begin{align*}
\inf_{m \in \mathcal{M}} \{ \mathcal{R}_n(\Proj_{\mathcal{F}_m}y) + o_m(1) \}.
\end{align*}
This type of oracle has access to $f^*$ and chooses the $m$ that minimizes the risk criterion up to a remainder term.

To derive this we finish the specification of $\Crit(m)$ by providing an appropriate prior $p_m$. The intuition behind our choice is the following. Defining $\hat{r}_m =\|y-\hat{f}_m\|_2^2$ and $\pen(m)=2\sigma^2 \ln \frac{1}{p_{m}} +\sigma^2 (d_m'+1)\ln\frac{N}{d_m'}$ we see that the criterion~\eqref{Crit equation} is of the form:
\begin{align*}
\Crit(m)=\hat{r}_m + \pen(m).
\end{align*}
The number of models in the family $\mathcal{M}$ having the same values of $d_m'$ and $d_m''$ grows exponentially with those dimensions. Thus for fix $d_m'$ and $d_m''$ we might find a model with low $\hat{r}_m$ just because of randomness since some of them will deviate largely from their means, which would correspond to an over-fitting case. Therefore, we need to penalize models of high dimensions more by taking into account the number of models with same dimensions. On the other hand we want this penalty to be as small as possible this way we give more importance to the fitting term $\hat{r}_m$. In particular we would prefer the term $2\sigma^2 \ln \frac{1}{p_{m}}$ to stay close to $\sigma^2 (d'+1)\ln\frac{N}{d'}$ at least for values of $d_m'$ close to $d_m''$. 

Our choice for $p_m$, useful inequalities and a complete discussion of the role of $p_m$ as a prior and tuning parameter for the risk can be found in Appendix D. From Lemmas \ref{lem:B_N} and \ref{important bounds}, the following oracle inequality can be derived for $\hat{f}_{\hat{m}}$:

\begin{theorem}[Oracle inequality for $\hat{f}_{\hat{m}}$]\label{thm:1}
With $\mathcal{M}$ restricted to models such that $ed'_m \leq N$ and for the choice of $m  \in \mathcal{M}$   corresponding to
\begin{align}
\hat{m} &\in \arg \min \limits_{m \in \mathcal{M}} \|y-\hat{f}_m\|_2^2 + \sigma^2K\pen(m), \\
\pen(m) &:= 2\ln \frac{1}{p_{m}} +(d'+1)\ln\frac{N}{d'}, \label{eq:pen}
\end{align}
with $K=3a$, we obtain for all $a>1$,
\begin{multline}\label{oracle inequality} 
E_{f^*}[\|  \Proj \limits_{\mathcal{F}_{\hat{m}} }Y  - f^* \|^2]  \leq \arg\min \limits_{m \in \mathcal{M}} \bigg\{ \frac{a}{a-1} \mathbb{E}_{f^*}[\|  \Proj \limits_{\mathcal{F}_{m} }Y - f^*\|^2] \\
+ \frac{a^2 \sigma^2}{a-1} \left(7 + 3(d_m'+1)\ln\frac{N}{d_m'}+6\ln\frac{1}{p_m} \right) \bigg\}.
\end{multline}
\end{theorem}
\begin{proof}
See Appendix~D.
\end{proof}
By investigating the oracle inequality, one notices that for an optimal choice of $a$ one has to make a trade-off between the performance of the oracle part and the bias part of the inequality. In general this trade-off is not possible to optimize since the value of the oracle part is not available to us and depends on the variance of the noise. In practice, one can use the SLOPE heuristic introduced in \citet{Lebarbier02} and described in \citet{Baudry12} and  in \cite{Arlot09}. In our case, the value of the tuning parameter can be chosen independently of the variance of the noise and we can use the value of $a$ for which we know that our estimator $\hat{f}_{\hat{m}}$ will perform well.
\begin{corollary}\label{useful regime}
For the set of models described in ~\eqref{regression} with $f^* \in \mathcal{F}_{m^*}$ the following properties hold:
\begin{itemize}
\item Adaptation and Risk Upper bound: The following adaptive upper bound in terms of $d_{m^*}'$ and $d_{m^*}''$ holds for $a=2$:
{\small \begin{align*}
&E_{f^*}[\|  \Proj \limits_{\mathcal{F}_{\hat{m}} }Y  - f^* \|^2] \leq 4 \sigma^2 \bigg( 7 + 3(d_{m^*}'+1)\ln\frac{N}{d_{m^*}'} \\ 
& \qquad \qquad +6\bigg(d_{m^*}' \ln[d_{m^*}''e^{\frac{13}{6}}] + d_{m^*}'' \ln[d_{m^*}'e^2] + d_{m^*}'' \ln \frac{N}{d_{m^*}''} \bigg) \bigg).
\end{align*}}
\item Consistency: If $d_{m^*}'' = o(N / \ln N)$, then: $$\lim_{N \to \infty } \frac{1}{N} \mathbb{E}_{f^*}[\| \hat{f}_{\hat{m}}- f^*\|^2] = 0. $$
\end{itemize}
\end{corollary}
\begin{proof}
See Appendix~D.
\end{proof}
We notice that the consistency condition $d_{m^*}'' = o(N / \ln N)$ is within the restriction on the models in theorem \ref{thm:1}, hence there is no loss of generality of having only models with $ed'_m \leq N$ in $\mathcal{M}$ since for other models we cannot guarantee convergent mean square risk anyway. In the next section we validate these theoretical guarantees by a series of tests on simulated data to get a sense of how tight the oracle inequality is, which signals  are difficult to estimate and how the algorithm behaves in practice. 

\section{Experimental results}
Consider first an experiment based data generated randomly according to the setup of \eqref{regression} with the same change points of Example~1. This is considered to be an easy case since $d'_{m^*}=4 < d''_{m^*}= 12 \ll N = 2000$, which is within the range of signals for which the consistency result of Corollary~\ref{useful regime} holds.

\begin{figure}[h] 
\begin{center}
\centerline{\includegraphics[width=\columnwidth]{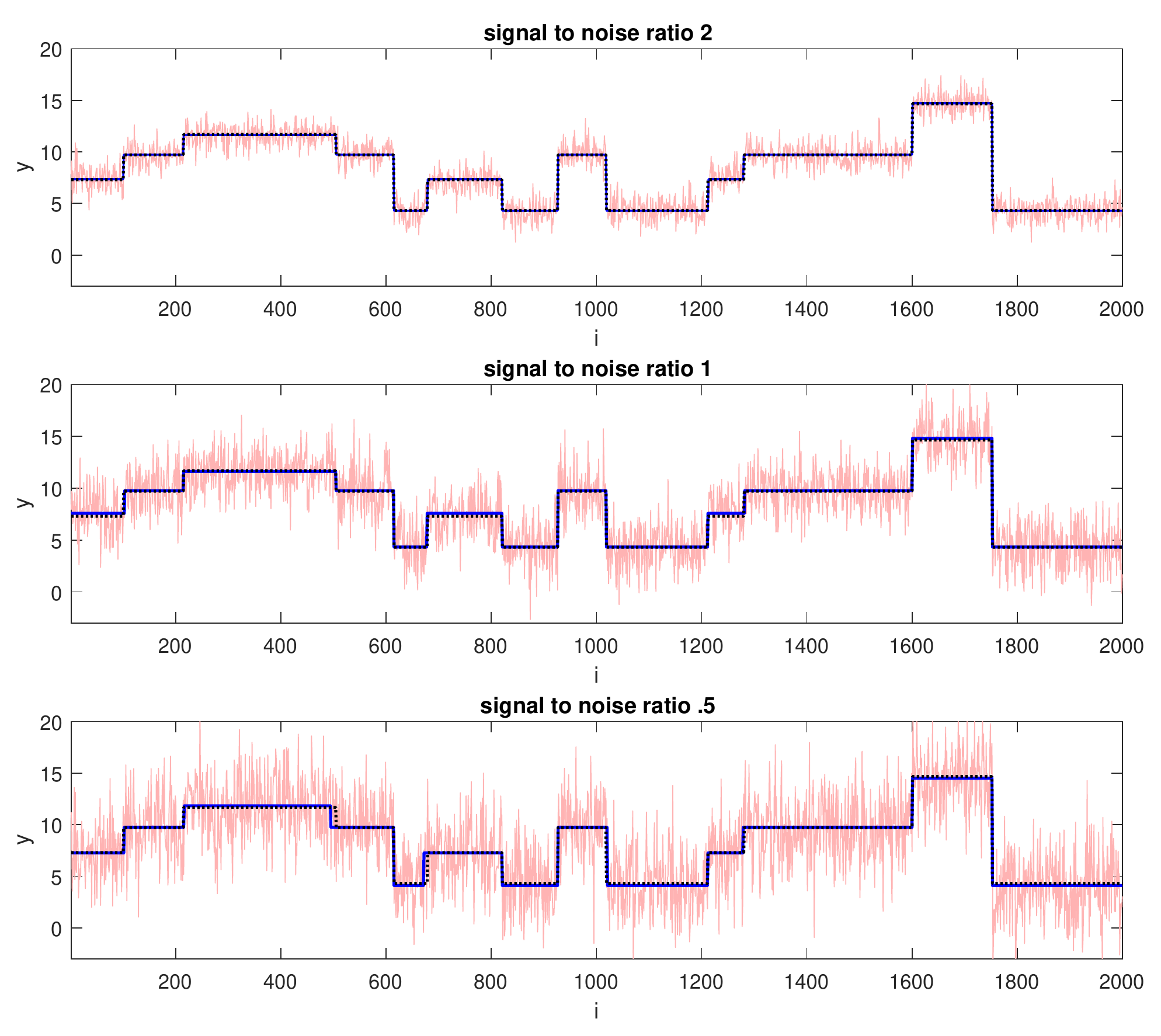}}
\caption{Estimates $\hat f$ (blue line) of $f^*$ (doted black line) obtained by Algorithm~\ref{two passes Dynamic programming algorithm} using the observed signal $Y$ (pink line), with $3$ different levels of signal-to-noise ratio.}
\label{fig2}
\end{center}
\vskip -0.2in
\end{figure}
The experiments in Figure~\ref{fig2} show that the algorithm is quite robust to the level of noise as measured by the signal-to-noise ratio $S/N = \frac{\textmd{magnitude of smallest jump in }f^*}{\sigma^2}$. We observe that the difference between the ground truth $f^*$ and $\hat{f}_{\hat{m}}$ is quite small even for small $S/N$ levels such as $S/N=0.5$ and the change point locations do not vary appreciably; in fact, for this experiment, $S/N =0.3$ seems to be the limiting case for which the algorithm performs well, and for lower values the risk upper-bound in Corollary~\ref{useful regime} becomes loose when $\sigma$ increases. Also, we note that an $S/N$ of $0.5$ is quite low for these kind of problems. In particular, algorithms relaying on the $L_1$-penalty such as Fussed LASSO do not achieve this kind of performance on the simpler task of change point only detection, while on the other hand they are more computational efficient \citep{Xin14}.

\begin{figure}[h]
\begin{center}
\centerline{\includegraphics[width=\columnwidth]{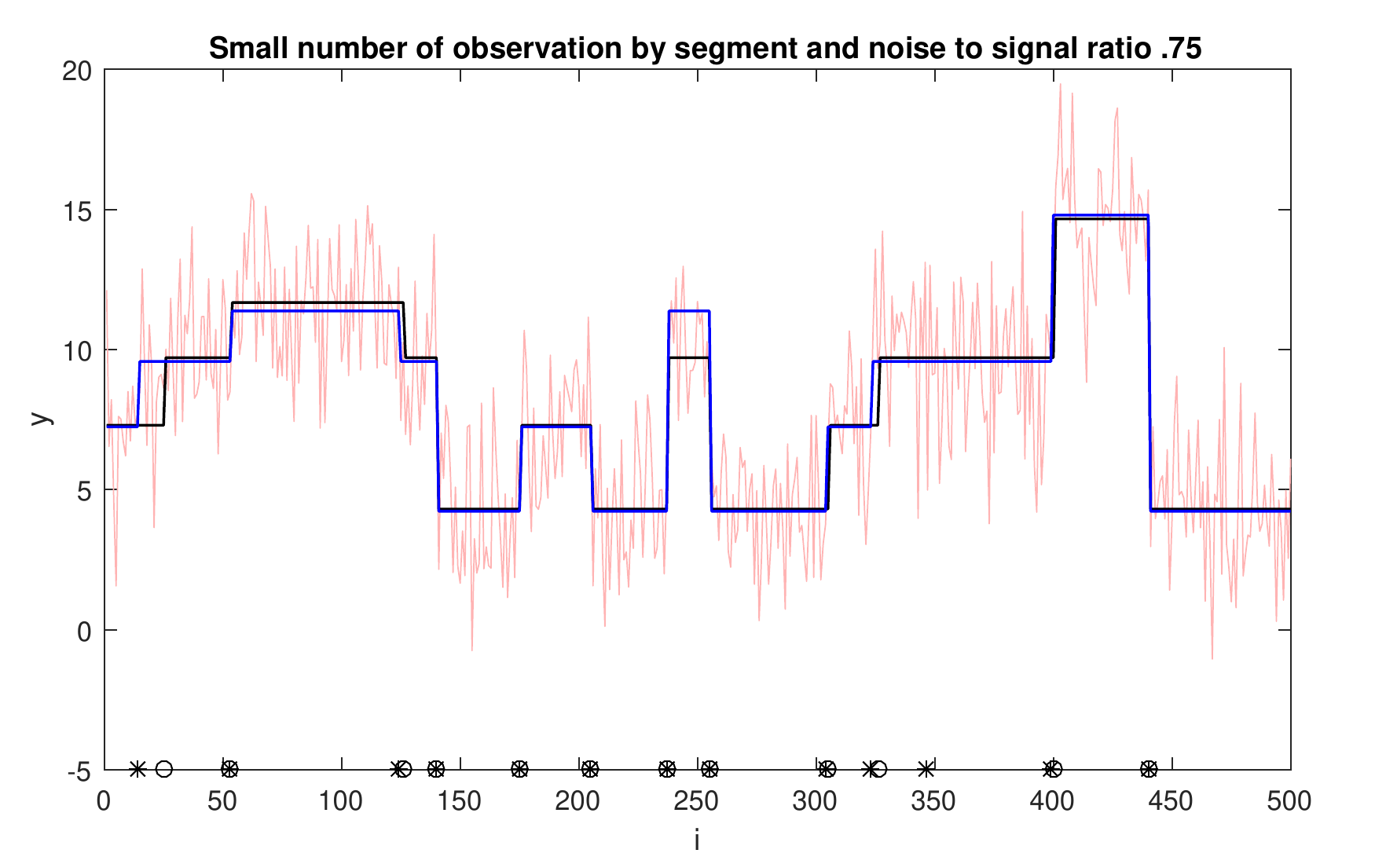}}
\caption{Estimate $\hat f$ (blue line) of $f^*$ (black line) obtained by Algorithm~\ref{two passes Dynamic programming algorithm} from a difficult observation sample $Y$ (pink line) with high signal-to-noise ratio ($1.5$) and few observations per segment ($N=500$ and $d''_{m^*}= 13$).}
\label{fig3}
\end{center}
\vskip -0.2in
\end{figure}
Figure~\ref{fig3} illustrates a difficult case, where we reduced the number of observation by segment by scaling down the signal $f^*$ to a support of size $N=500$. Now we are outside of the useful regime of Corollary~\ref{useful regime} and we notice that the second segment $\llbracket 15,53 \rrbracket$ is wider than what it should since the first change point at $25$ was detected at $14$; also the segment $\llbracket 206,237\rrbracket$ belongs to cluster $[4]$ while it is actually in cluster $[3]$ in the original signal $f^*$. Nevertheless we can observe an interesting property for segment $\llbracket 324,346 \rrbracket$, namely, that the end point $346$ does not correspond to any real change point, yet this segment belongs to the optimal solution of the $1^{st}$ dynamic programming pass. On the other hand the $2^{nd}$ dynamic programming pass puts it in the same cluster $[3]$ as $\llbracket 347,399 \rrbracket$, turning them into one single segment of cluster $[3]$. This behavior actually is the norm for the algorithm, where false changes are often detected in difficult signals in the $1^{st}$ dynamic programming pass but are removed after the $2^{nd}$ pass. These kinds of false discoveries are actually one of the weaknesses of many change point only detection algorithms like Fussed LASSO, and they have been studied in \citep{Harchaoui07b}, \citep{Rinaldo09} and \citep{Rojas14}. 
\begin{figure}[h]
\begin{center}
\centerline{\includegraphics[width=\columnwidth]{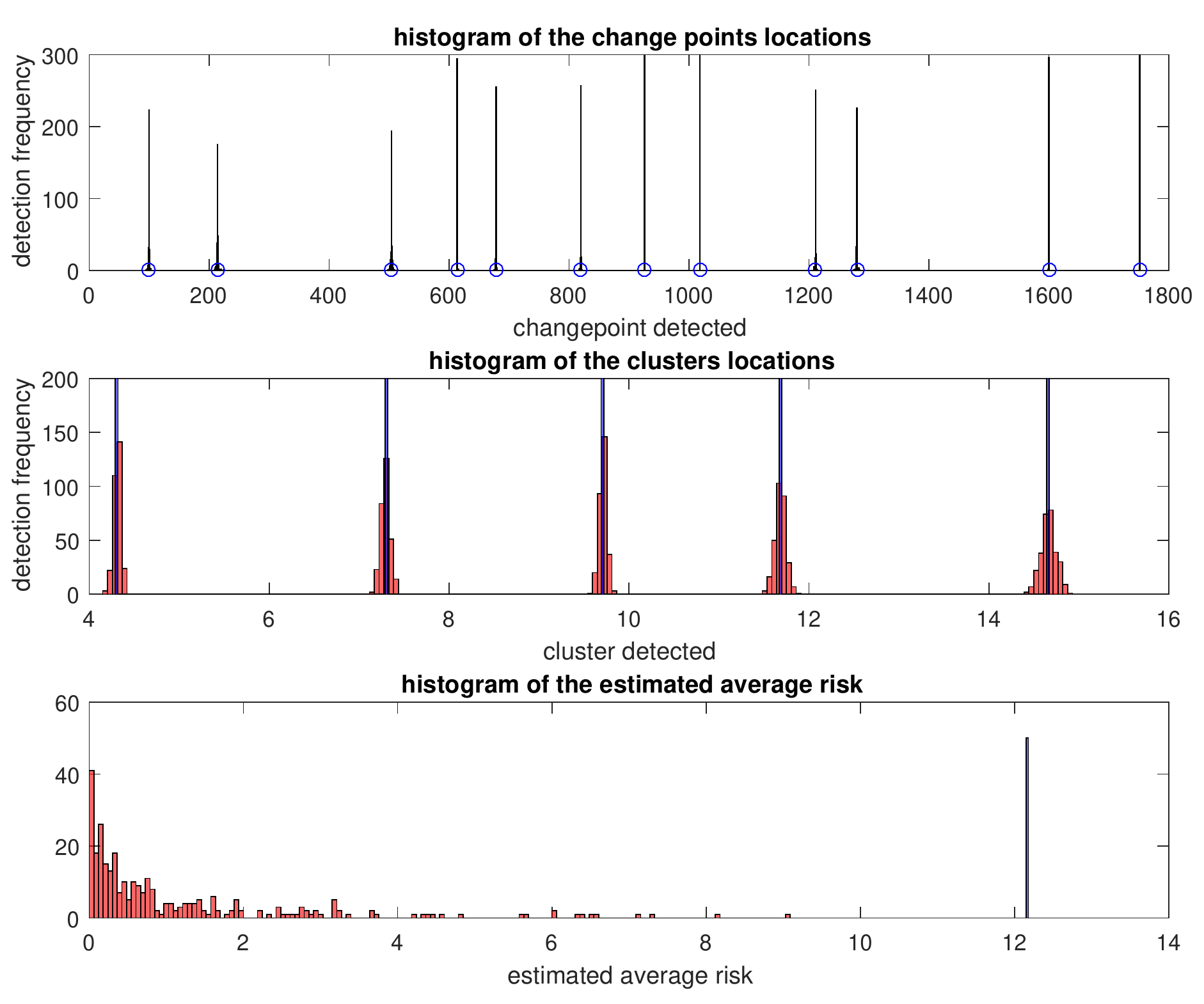}}
\caption{Top histogram: location of estimated (black) and true (black) change points in 300 simulations. Middle histogram: location of estimated (red) and true (blue) clusters in 300 simulations. Bottom histogram: 300 estimates of the average mean square error (red) and its theoretical upper bound (blue).}
\label{fig4}
\end{center}
\vskip -0.2in
\end{figure}
In the last experiment, we run Algorithm~\ref{two passes Dynamic programming algorithm} 300 times with the parameter values $d'_{m^*}=4 < d''_{m^*}= 12 \ll N=2000$ and signal-to-noise ratio $S/N = 1$; Figure~\ref{fig4} summarizes the results. In the top histogram we notice that the algorithm successfully detects the change points most of the time; in fact, the achieved accuracy was $\frac{\textrm{number of change points correctly detected}}{\textrm{number of change points detected}} \approx 0.8528$. The middle histogram shows the placement of estimated clusters and the true values of the clusters; we observe that the true values lie in a small neighborhood of the estimated values for every cluster. In the bottom histogram we observe that the theoretical upper bound on the average mean square error --in this case $12.1575$-- found in Corollary~\ref{useful regime} is very conservative and most of the 300 estimates --given by $\frac{\|\hat{f}_{\hat{m}}-f^*\|^2}{N}$-- are significantly smaller.

\section{Conclusions}
In this work we considered a novel problem related to change point detection where we have to address the simultaneous task of segmenting and clustering the observed signal. Our approach has been to view this problem as a non parametric model selection problem on the set of all possible partitions. We derived for this the computationally tractable Algorithm~\ref{two passes Dynamic programming algorithm}, that computes a relaxation of the penalized minimization of criterion~\eqref{Criterion}, and we justified it from a statistical standpoint by showing that this minimization can be viewed as an approximate MAP. This approximate MAP estimate enjoys the properties of being adaptive and consistent in the sense of Corollary~\ref{useful regime}. We finally justified the use of Algorithm~\ref{two passes Dynamic programming algorithm} by simulation data that shows some useful properties of the resulting estimate and validates the theoretical guarantees.

One extension of this work concerns developing a more complete analysis of Algorithm~\ref{two passes Dynamic programming algorithm}, to obtain consistency results on the number and locations of the change points and clusters. Another possible extension relates to the use of Algorithm~\ref{two passes Dynamic programming algorithm} in the non scalar case; this was already explored for change point only detection in \cite{Arlot16} through the use of characteristic kernels~\citep{Sriperumbudur11}. We believe that the same approach can be adopted here except that we cannot perform the sorting step; this can be overcome using a Kernel clustering algorithm~\cite{Filipponea08} or a spectral version of it \cite{Schölkopf98} for the second stage. Finally, the remark after Figure~\ref{fig3} hints to the possibility of using a combined algorithm starting with the sparse solution of Fussed LASSO and running the $2^{nd}$ dynamic programming pass of our algorithm as a way to boost the performance of Fussed LASSO to get rid of false discoveries. This would be still computationally attractive according to the comment after Theorem~\ref{computaional complexity}, since the solution of Fussed LASSO has a small number of changes.  
\newpage

\bibliographystyle{plainnat}
\bibliography{sample}

\begin{thebibliography}{36}
\providecommand{\natexlab}[1]{#1}
\providecommand{\url}[1]{\texttt{#1}}
\expandafter\ifx\csname urlstyle\endcsname\relax
  \providecommand{\doi}[1]{doi: #1}\else
  \providecommand{\doi}{doi: \begingroup \urlstyle{rm}\Url}\fi

\bibitem[Abou-Elailah et~al.(2016)Abou-Elailah, Gouet-Brunet, and
  Bloch]{Abou-Elailah16}
A.~Abou-Elailah, V.~Gouet-Brunet, and I.~Bloch.
\newblock Detection of abrupt changes in spatial relationships in video
  sequences.
\newblock In \emph{International Conference on Pattern Recognition Applications
  and Methods}, pages 89--106, 2016.

\bibitem[Arlot and Massart(2009)]{Arlot09}
S.~Arlot and P.~Massart.
\newblock Data-driven calibration of penalties for least-squares regression.
\newblock \emph{Journal of Machine Learning Research}, 10:\penalty0 245--279,
  2009.

\bibitem[Arlot et~al.(2016)Arlot, Celisse, and Harchaoui]{Arlot16}
S.~Arlot, A.~Celisse, and Z.~Harchaoui.
\newblock Kernel change-point detection.
\newblock \emph{ArXiv:1202.3878v1}, 2016.

\bibitem[Arnold and Tibshirani(2016)]{Arnold16}
T.~B. Arnold and R.~J. Tibshirani.
\newblock Efficient implementations of the generalized lasso dual path
  algorithm.
\newblock \emph{Journal of Computational and Graphical Statistics}, 25\penalty0
  (1):\penalty0 1–--27, 2016.

\bibitem[Baudry et~al.(2012)Baudry, Maugis, and Michel]{Baudry12}
J.~P. Baudry, C.~Maugis, and B.~Michel.
\newblock Slope heuristics: overview and implementation.
\newblock \emph{Statistics and Computing}, 22:\penalty0 455–--470, 2012.

\bibitem[Boucheron et~al.(2013)Boucheron, Lugosi, and Massart]{Boucheron13}
S.~Boucheron, G.~Lugosi, and P.~Massart.
\newblock \emph{Concentration Inequalities: A Nonasymptotic Theory of
  Independence}.
\newblock Oxford University Press, 2013.

\bibitem[Celisse et~al.(2017)Celisse, Marot, Pierre-Jean, and
  Rigaill]{Celisse17}
A.~Celisse, G.~Marot, M.~Pierre-Jean, and G.~Rigaill.
\newblock New efficient algorithms for multiple change-point detection with
  kernels.
\newblock \emph{arXiv:1710.04556}, 2017.

\bibitem[Cirel'son et~al.(1976)Cirel'son, Ibragimov, and Sudakov]{Cirelson76}
B.~S. Cirel'son, I.~A. Ibragimov, and V.~N. Sudakov.
\newblock Norms of {Gaussian} sample functions.
\newblock In \emph{Proceedings of the Third Japan-USSR Symposium on Probability
  Theory}, pages 20--41, 1976.

\bibitem[Cleynen and Lebarbier(2014)]{Cleynen14}
A.~Cleynen and E.~Lebarbier.
\newblock Segmentation of the {Poisson} and negative binomial rate models: a
  penalized estimator.
\newblock \emph{Probability and Statistics}, 18\penalty0 (2):\penalty0
  750–--769, 2014.

\bibitem[Cormen et~al.(2009)Cormen, Leiserson, Rivest, and Stein]{CLRS09}
T.~H. Cormen, C.~E. Leiserson, R.~L. Rivest, and C.~Stein, editors.
\newblock \emph{Introduction to Algorithms, Third Edition}.
\newblock MIT Press, Cambridge, MA, USA, 2009.

\bibitem[Dalalyan et~al.(2017)Dalalyan, Hebiri, and Lederer]{Dalalyan17}
A.~S. Dalalyan, M.~Hebiri, and J.~Lederer.
\newblock On the prediction performance of the {Lasso}.
\newblock \emph{Bernoulli}, 23\penalty0 (1):\penalty0 552--581, 2017.

\bibitem[Filipponea et~al.(2008)Filipponea, Camastrab, Masullia, and
  Rovetta]{Filipponea08}
M.~Filipponea, F.~Camastrab, F.~Masullia, and S.~Rovetta.
\newblock A survey of kernel and spectral methods for clustering.
\newblock \emph{Journal of Machine Learning Research}, 41\penalty0
  (1):\penalty0 176--190, 2008.

\bibitem[Garreau and Arlot(2017)]{Garreau17}
D.~Garreau and S.~Arlot.
\newblock Consistent change-point detection with kernels.
\newblock \emph{arXiv:1612.04740}, 2017.

\bibitem[Goodrich and Tamassia(2001)]{Goodrich01}
M.~T. Goodrich and R.~Tamassia, editors.
\newblock \emph{Algorithm Design: Foundations, Analysis, and Internet
  Examples}.
\newblock John Wiley and Sons, Hoboken, NJ, USA, 2001.

\bibitem[Graham et~al.(1988)Graham, Knuth, and Patashnik]{ConcreteMathematics}
R.~L. Graham, D.~E. Knuth, and O.~Patashnik.
\newblock \emph{Concrete Mathematics}.
\newblock Addison–Wesley, 1988.

\bibitem[Harchaoui and Capp\'e(2007)]{Harchaoui07a}
Z.~Harchaoui and O.~Capp\'e.
\newblock Retrospective multiple change-point estimation with kernels.
\newblock In \emph{14th IEEE/SP Workshop on Statistical Signal Processing (SSP
  '07)}, Madison, WI, USA, 2007.

\bibitem[{H\"utter} and Rigollet(2016)]{Hütter16}
J.~{H\"utter} and P.~Rigollet.
\newblock Detection of abrupt changes in spatial relationships in video
  sequences.
\newblock In \emph{29th Annual Conference on Learning Theory, PMLR 49}, pages
  1115--1146, 2016.

\bibitem[Jackson et~al.(2005)Jackson, Scargle, Barnes, Arabhi, Alt, Gioumousis,
  Gwin, Sangtrakulcharoen, Tan, and Tsai]{Jackson05}
B.~Jackson, J.~D. Scargle, D.~Barnes, S.~Arabhi, A.~Alt, P.~Gioumousis,
  E.~Gwin, P.~Sangtrakulcharoen, L.~Tan, and T.~T. Tsai.
\newblock An algorithm for optimal partitioning of data on an interval.
\newblock \emph{IEEE Signal Processing Letters}, 12\penalty0 (2):\penalty0
  105--–108, 2005.

\bibitem[Kim et~al.(2009)Kim, Marzban, Percival, and Stuetzle]{Kim09}
A.~Y. Kim, C.~Marzban, D.~B. Percival, and W.~Stuetzle.
\newblock Using labeled data to evaluate change detectors in a multivariate
  streaming environment.
\newblock \emph{Signal Processing}, 89\penalty0 (12):\penalty0 2529–--2536,
  2009.

\bibitem[Lai et~al.(2005)Lai, Johnson, Kucherlapati, and Park]{Lai05}
W.~R. Lai, M.~D. Johnson, R.~Kucherlapati, and P.~J. Park.
\newblock Comparative analysis of algorithms for identifying amplifications and
  deletions in array {CGH} data.
\newblock \emph{Bioinformatics}, 21:\penalty0 3763--–3770, 2005.

\bibitem[Lavielle and Teyssière(2006)]{Lavielle06}
M.~Lavielle and G.~Teyssière.
\newblock Detection of multiple change-points in multivariate time series.
\newblock \emph{Lithuanian Mathematical Journal}, 46:\penalty0 287--–306,
  2006.

\bibitem[Lebarbier(2002)]{Lebarbier02}
E.~Lebarbier.
\newblock \emph{Quelques approches pour la d'etection de ruptures a horizon
  fini}.
\newblock PhD thesis, Faculté des Sciences d'Orsay (Essonne), Universite
  Paris-Sud, 2002.

\bibitem[Levy-leduc and Harchaoui(2008)]{Harchaoui07b}
C.~Levy-leduc and Z.~Harchaoui.
\newblock Catching change-points with lasso.
\newblock In \emph{Advances in Neural Information Processing Systems 20}, pages
  617--624, 2008.

\bibitem[Massart(2003)]{Massart03}
P.~Massart.
\newblock \emph{Concentration Inequalities and Model Selection}.
\newblock Springer-Verlag Berlin Heidelberg, 1st edition, 2003.

\bibitem[Mohr and Porter(2009)]{AMohr}
A.~Mohr and T.~D. Porter.
\newblock Applications of chromatic polynomials involving {Stirling} numbers.
\newblock \emph{Journal of Combinatorial Mathematics and Combinatorial
  Computing}, 70:\penalty0 57--–64, 2009.

\bibitem[Rennie and Dobson(1969)]{Rennie69}
B.~C. Rennie and A.~J. Dobson.
\newblock On {Stirling} numbers of the second kind.
\newblock \emph{Journal of Combinatorial Theory}, 7\penalty0 (2):\penalty0
  116--121, 1969.

\bibitem[Rigaill et~al.(2012)Rigaill, Lebarbier, and Robin]{Rigaill12}
G.~Rigaill, E.~Lebarbier, and S.~Robin.
\newblock Exact posterior distributions and model selection criteria for
  multiple change-point detection problems.
\newblock \emph{Statistics and Computing}, 22\penalty0 (4):\penalty0
  917--–929, 2012.

\bibitem[Rinaldo(2009)]{Rinaldo09}
A.~Rinaldo.
\newblock Properties and refinements of the fused lasso.
\newblock \emph{The Annals of Statistics}, 37\penalty0 (5b):\penalty0
  2922--–2952, 2009.

\bibitem[Rojas and Wahlberg(2014)]{Rojas14}
C.~R. Rojas and B.~Wahlberg.
\newblock On change point detection using the fused lasso method.
\newblock \emph{arXiv:1401.5408}, 2014.

\bibitem[Rudin et~al.(1992)Rudin, Osher, and Fatemi]{Rudin92}
L.~I. Rudin, S.~Osher, and E.~Fatemi.
\newblock Nonlinear total variation based noise removal algorithms.
\newblock \emph{Physica D: Nonlinear Phenomena}, 60\penalty0 (1):\penalty0
  259--–268, 1992.

\bibitem[{Sch\"olkopf} et~al.(1998){Sch\"olkopf}, Smola, and
  {M\"uller}]{Schölkopf98}
B.~{Sch\"olkopf}, A.~Smola, and K.~R. {M\"uller}.
\newblock Nonlinear component analysis as a kernel eigenvalue problem.
\newblock \emph{Neural Computation}, 10\penalty0 (5):\penalty0 1299--1319,
  1998.

\bibitem[Spokoiny(2009)]{Spokoiny09}
V.~Spokoiny.
\newblock Multiscale local change point detection with applications to
  value-at-risk.
\newblock \emph{The Annals of Statistics}, 37:\penalty0 1405--1436, 2009.

\bibitem[Sriperumbudur et~al.(2011)Sriperumbudur, Fukumizu, and
  Lanckriet]{Sriperumbudur11}
B.~K. Sriperumbudur, K.~Fukumizu, and G.~R.~G. Lanckriet.
\newblock Universality, characteristic kernels and rkhs embedding of measures.
\newblock \emph{Journal of Machine Learning Research}, 12:\penalty0 2389--2410,
  2011.

\bibitem[Tartakovsky et~al.(2014)Tartakovsky, Nikiforov, and
  Basseville]{Tartakovsky14}
A.~Tartakovsky, I.~V. Nikiforov, and M.~Basseville.
\newblock \emph{Hypothesis Testing and Changepoint Detection}.
\newblock Chapman and Hall, Monographs on Statistics and Applied Probability,
  2014.

\bibitem[Tibshirani et~al.(2005)Tibshirani, Saunders, Rosset, Zhu, and
  Knight]{Tibshirani05}
R.~Tibshirani, M.~Saunders, S.~Rosset, J.~Zhu, and K.~Knight.
\newblock Sparsity and smoothness via the fused lasso.
\newblock \emph{Journal of the Royal Statistical Society: Series B (Statistical
  Methodology)}, 67\penalty0 (1):\penalty0 91–--108, 2005.

\bibitem[Xin et~al.(2014)Xin, Kawahara†, Wang, and Gao]{Xin14}
B.~Xin, Y.~Kawahara†, Y.~Wang, and W~Gao.
\newblock Efficient generalized fused lasso and its application to the
  diagnosis of alzheimer’s disease.
\newblock In \emph{Proceedings of the Twenty-Eighth AAAI Conference on
  Artificial Intelligence}, 2014.

\end{thebibliography}

\newpage

\onecolumn

\section*{Appendix}
In this appendix we provide proofs of all the technical derivations. For the reader convenience the corresponding lemmas and theorems are restated. 

\subsection*{Appendix A: Normal Projections}

The following standard lemma on normal projections will be needed later for computing expectations of projection estimators and linking them to the dimension of the corresponding subspace.

\begin{lemma} \label{gaussian projection lemma}
Let $\epsilon$ be a random $\mathcal{N}(0,I_N)$ vector in $\mathbb{R}^N$, and let $\mathcal{S}$ be a linear subspace in $\mathbb{R}^N$. Then the random vector $\Proj_{\mathcal{S}}\epsilon = U_{\mathcal{S}} \epsilon$ has an $\mathcal{N}(0,U_{\mathcal{S}})$ distribution with $U_{\mathcal{S}}$ being the projection matrix onto $\mathcal{S}$. The squared norm $\|\Proj_{\mathcal{S}}\epsilon\|^2$ is $\chi^2_{d}$-distributed with $d=\dim(\mathcal{S})$. In particular $\mathbb{E}[\|\Proj_{\mathcal{S}}\epsilon\|^2]=d$.
\end{lemma}

\begin{proof}
By definition of Gaussian vectors, for every $a \in \mathbb{R}^N$ we have that
\begin{align*}
\mathbb{E}[\exp(\inner{a}{U_{\mathcal{S}}\epsilon}]=\mathbb{E}[\exp(\inner{U_{\mathcal{S}}^Ta}{\epsilon})]=\exp(\frac{1}{2}\inner{a}{U_{\mathcal{S}}U_{\mathcal{S}}^Ta}),
\end{align*}
and since $U_{\mathcal{S}}$ is a projection matrix, $U_{\mathcal{S}}U_{\mathcal{S}}^T = U_{\mathcal{S}}$, so $U_{\mathcal{S}}\epsilon \sim \mathcal{N}(0,U_{\mathcal{S}})$. Let $\{v_1,\dots,v_d\}$ be an orthonormal basis for $\mathcal{S}$ and $V=[v_1,\dots,v_d]$; then $V^TV=I_d$ and $V\epsilon \sim \mathcal{N}(0,I_d)$. Also,
\begin{align*}
\|\Proj_{\mathcal{S}}\epsilon\|^2=\sum_{i=1}^d(v_i^T\epsilon)^2=\|V\epsilon\|^2,
\end{align*}
thus $\|\Proj_{\mathcal{S}}\epsilon\|^2$ is $\chi^2$-distributed with $d$ degrees of freedom, so
\begin{align*}
\mathbb{E}[\|\Proj_{\mathcal{S}}\epsilon\|^2]=d.
\end{align*}
\end{proof}

\subsection*{Appendix B: Two-Pass Dynamic Programming}
\begin{proof}[Proof of \eqref{equation for the algorithm}]
Since we defined $\pi_{\bar{m}}$ to be the partition having as elements all the segments of $\pi_{m}$, in particular we have $m \subset \bar{m}$ and $$\Proj_{\mathcal{F}_{m}} = \Proj_{\mathcal{F}_{\bar{m}}} \Proj_{\mathcal{F}_{m}} =\Proj_{\mathcal{F}_{m}} \Proj_{\mathcal{F}_{\bar{m}}}. $$
Moreover, since $\Proj_{\mathcal{F}_{\bar{m}}} Y - \Proj_{\mathcal{F}_{m}} Y \in \mathcal{F}_{\bar{m}}$, then by the projection theorem we have that $$ (Y - \Proj_{\mathcal{F}_{\bar{m}}} Y) \bot (\Proj_{\mathcal{F}_{\bar{m}}} Y - \Proj_{\mathcal{F}_{m}} Y),$$ 
so the Pythagorean theorem implies that
\begin{equation*}
\|Y - \Proj_{\mathcal{F}_{m}} Y\|^2 = \|Y - \Proj_{\mathcal{F}_{\bar{m}}} Y\|^2
+ \| \Proj_{\mathcal{F}_{\bar{m}}} Y - \Proj_{\mathcal{F}_{m}} Y\|^2.
\end{equation*}
Thus, the minimization of criterion~\eqref{Criterion} simplifies to
\begin{align*} 
\min \limits_{m \in \mathcal{M}} \Crit(m) &= \min \limits_{0 \leq d' \leq d'' \leq D} \bigg\{ \min \limits_{|\bar{m}|=d''} \Bigg\{ \|Y - \Proj_{\mathcal{F}_{\bar{m}}} Y\|^2  \\
&+ \min \limits_{\substack{m \subset \bar{m}\\|m|=d'}} \| \Proj_{\mathcal{F}_{\bar{m}}} Y - \Proj_{\mathcal{F}_{m}}\Proj_{\mathcal{F}_{\bar{m}}}  Y\|^2 \Bigg\} + \sigma^2K\pen(d',d'') \Bigg\}.
\end{align*}
Instead of computing this minimum exactly we will take a greedy step in the second minimum by defining 
$$\tilde{m} := \arg \min_{|\bar{m}|=d''} \|Y - \Proj_{\mathcal{F}_{\bar{m}}} Y\|^2,$$
and plugging it into the third minimum to obtain the following relaxation:
\begin{align*}
\min \limits_{m \in \mathcal{M}} \Crit(m) &\leq \min \limits_{0 \leq d' \leq d'' \leq D} \Bigg\{  \|Y - \Proj_{\mathcal{F}_{\tilde{m}}} Y\|^2  \\
& \qquad+ \min \limits_{\substack{m \subset \tilde{m}\\|m|=d'}} \| \Proj_{\mathcal{F}_{\tilde{m}}} Y - \Proj_{\mathcal{F}_{m}}\Proj_{\mathcal{F}_{\tilde{m}}}  Y\|^2  + \sigma^2K\pen(d',d'') \Bigg\}  \\
&= \min \limits_{0 \leq d' \leq d'' \leq D} \Bigg\{ \min \limits_{|m|=d''} \|Y - \Proj_{\mathcal{F}_{m}} Y\|^2  \\
&\qquad + \min \limits_{\substack{m \subset \tilde{m}\\|m|=d'}} \| \Proj_{\mathcal{F}_{\tilde{m}}} Y - \Proj_{\mathcal{F}_{m}}\Proj_{\mathcal{F}_{\tilde{m}}}  Y\|^2 + \sigma^2K\pen(d',d'') \Bigg\}. 
\end{align*}
The second inner minimization of the last equation can then be relaxed by restricting it to partitions satisfying the clustering property, since
\begin{equation*}
\min \limits_{m \in \mathcal{M}_{\tilde{m},d'}} \| \Proj_{\mathcal{F}_{\tilde{m}}} Y - \Proj_{\mathcal{F}_{m}}\Proj_{\mathcal{F}_{\bar{m}}}  Y\|^2 \leq \min \limits_{m \in \mathcal{M}_{\bar{y}_{\tilde{m}},d'}} \| \Proj_{\mathcal{F}_{\tilde{m}}} Y - \Proj_{\mathcal{F}_{m}}\Proj_{\mathcal{F}_{\bar{m}}}  Y\|^2.
\end{equation*}
This leads to the following upper bound:
\begin{align*}
\min \limits_{m \in \mathcal{M}} \Crit(m)  &\leq \min \limits_{0 \leq d'' \leq D} \bigg\{ \min \limits_{|m|=d''} \|Y - \Proj_{\mathcal{F}_{m}} Y\|^2  \\
&+ \min \limits_{\substack{0 \leq d' \leq d'' \\ m \in \mathcal{M}_{\bar{y}_{\tilde{m}},d''}}} \| \Proj_{\mathcal{F}_{\tilde{m}}} Y -\Proj_{\mathcal{F}_{m}} \Proj_{\mathcal{F}_{\tilde{m}}} Y\|^2  + \sigma^2K\pen(d',d'') \bigg\}.
\end{align*}
Therefore, we can define the following relaxation for the minimization of the criterion in \eqref{Criterion}:
\begin{multline*} 
\Crit_r(d'') :=   \min \limits_{|m|=d''} \|Y - \Proj_{\mathcal{F}_{m}} Y\|^2 \\
+ \min \limits_{\substack{0 \leq d' \leq d''\\
m \in \mathcal{M}_{\bar{y}_{\tilde{m}},d''}}} \Bigg\{ \| \Proj_{\mathcal{F}_{\tilde{m}}} Y - \Proj_{\mathcal{F}_{m}}\Proj_{\mathcal{F}_{\tilde{m}}} Y\|^2  + \sigma^2K\pen(d'_m,d''_m) \bigg\},
\end{multline*}
which corresponds to \eqref{equation for the algorithm}.
\end{proof}
\begin{customthm}{4.1}

Let $(y_i)_{i=1}^N \subset \mathbb{R}$, $D \in \mathbb{N}$ and $K>0$. Then, recalling the dynamic programming recursions in \eqref{dynamic programming 1} and \eqref{dynamic programming 2},
\begin{itemize}
\item for all $1 \leq d \leq D$,
\begin{align*}
\tilde{m}_d \in \arg \min \limits_{|\bar{m}|=d} \|Y - \Proj_{\mathcal{F}_{\bar{m}}} Y\|^2,
\end{align*}
\item for all $1 \leq \delta \leq d \leq D$,
\begin{align*} 
\tilde{\tilde{m}}_{(d,\delta)} \in \arg \min \limits_{m \in \mathcal{M}_{\bar{y}_{\tilde{m}}\delta}} \| \Proj_{\mathcal{F}_{\bar{m}}} Y - \Proj_{\mathcal{F}_{m}}\Proj_{\mathcal{F}_{\bar{m}}}  Y\|^2.
\end{align*}  
\end{itemize}
Furthermore, Algorithm~\ref{two passes Dynamic programming algorithm} correctly solves the minimization problem in ~\eqref{equation for the algorithm}, with time and space complexity $\mathcal{O}(N^3+D^4)$ and $\mathcal{O}(N^2+D^3)$, respectively. 
\end{customthm}
\begin{proof}
Here we abuse the notation of $m$ and $\mathcal{M}$ so that if $Y^n$ is the sub-vector of the first $n$ component of $Y$ then $m$ and $\mathcal{M}$ are still defined by mere restriction to the first $n$ component and $\Proj_{\mathcal{F}_{m}} Y^n$ still makes sense for $m \in \mathcal{M}$.

To prove the $1^{st}$ point we need to show that $C_d(n)$, defined inductively as
\begin{align*}
C_1(n) &:= R_{[1,n]}, \qquad\qquad\qquad\qquad\qquad\qquad\qquad\quad\;\;\, n \in \llbracket 1,N \rrbracket, \\
C_d(n) &:= \min \limits_{i \in \llbracket d,n \rrbracket} \{ C_{d-1}(i-1) + R_{[i,n]}\}, \qquad 2 \leq d \leq D, \quad d \leq n \leq N,
\end{align*}
is equal to $\min \limits_{|\bar{m}|=d} \|Y^n - \Proj_{\mathcal{F}_{\bar{m}}} Y^n\|^2$, with $Y^n=(Y_1,\dots,Y_n)$. This implies that for $n=N$ we obtain the result for all $d$. This is straightforward since if, for $Y^n$, $\pi_m=\{0=i_0 < i_1<\dots <i_d < i_{d+1}=n\}$ is a partition, then
\begin{equation*}
\|Y^n - \Proj_{\mathcal{F}_{\bar{m}}} Y^n\|^2 = \sum \limits_{k=0}^{d} R_{[i_{k}+1,i_{k+1}]}.
\end{equation*}
Taking the minimum over $|\bar{m}|=d$ or, equivalently, over the values of $i_1, i_2, \dots, i_d$, we obtain
\begin{align*}
&\min \limits_{|\bar{m}| = d} \|Y^n - \Proj_{\mathcal{F}_{\bar{m}}} Y^n \|^2 \\&= \min \limits_{0=i_0 < i_1<\dots <i_d < i_{d+1}=n} \sum \limits_{k=0}^{d} R_{[i_{k}+1,i_{k+1}]} \\
&= \min \limits_{d \leq i_d+1 \leq n} \left\{ \min \limits_{0 < i_1<\dots <i_{d-1}<i_d} \left\{ \sum \limits_{k=0}^{d-1} R_{[i_{k}+1,i_{k+1}]} \right\} + R_{[i_d+1,n]} \right\} \\
&= \min \limits_{i \in \llbracket d,n \rrbracket} \{ C_{d-1}(i-1) + R_{[i,n]}\}.
\end{align*}
This yields our $1^{st}$ point:
\begin{align*}
\tilde{m}_d \in \arg \min \limits_{|\bar{m}|=d} \|Y - \Proj_{\mathcal{F}_{\bar{m}}} Y\|^2.
\end{align*}
To prove the second point, we first define, in the same notations of Algorithm~\ref{two passes Dynamic programming algorithm}, $Y_r$ as a rearrangement of $Y$ according to the permutation $\phi_d$, and $\tilde{m}_r$ as a rearrangement of $\tilde{m}$. Also $Y_r^{(t)}$ denotes the truncation of $Y_r$ to the $t^\text{th}$-segment. By the clustering property in ~\eqref{clutering property} we have that $m \in \mathcal{M}_{\bar{y}_{\tilde{m}_r},\delta} $ if and only if $\bar{m} \in \mathcal{M}_{\bar{y}_{\tilde{m}_r},\delta}$, hence
\begin{align*}
 \min \limits_{m \in \mathcal{M}_{\bar{y}_{\tilde{m}},\delta}} \| \Proj_{\mathcal{F}_{\bar{m}}} Y -\Proj_{\mathcal{F}_{m}} \Proj_{\mathcal{F}_{\bar{m}}}  Y\|^2
 &=  \min \limits_{m \in \mathcal{M}_{\bar{y}_{\tilde{m}_r},\delta}} \| \Proj_{\mathcal{F}_{\bar{m}_r}} Y_r - \Proj_{\mathcal{F}_{m}} \Proj_{\mathcal{F}_{\bar{m}_r}}  Y_r \|^2 \\
 &= \min \limits_{m \in \mathcal{M}_{\bar{y}_{\tilde{m}_r},\delta}} \| \Proj_{\mathcal{F}_{\bar{m}_r}} Y_r - \Proj_{\mathcal{F}_{\bar{m}}} \Proj_{\mathcal{F}_{\bar{m}_r}}  Y_r \|^2.
\end{align*}
Here we are back to the same setup of the $1^{st}$ point, so we need to show that $G_{(t,\delta)}$ defined inductively as
\begin{align*}
G_{(t,1)} &:=  \bar{R}_{[1,t]}, \qquad\qquad\qquad\qquad\qquad\qquad\qquad\quad\;\;\, t \in \llbracket 1,d \rrbracket, \\
G_{(t,\delta)} &:= \min \limits_{i \in \llbracket \delta,t \rrbracket} \{ G_{(i-1,\delta-1)} + \bar{R}_{[i,t]}\}, \qquad \qquad \qquad 2 \leq \delta \leq t \leq d,
\end{align*}
is equal to 
$ \min_{\bar{m} \in \mathcal{M}_{\bar{y}_{\tilde{m}_r},\delta}} \| \Proj_{\mathcal{F}_{\bar{m}_r}} Y_r^t - \Proj_{\mathcal{F}_{\bar{m}}} \Proj_{\mathcal{F}_{\bar{m}_r}}  Y_r^t \|^2$ and we obtain the desired result at $t=d$.
The same derivation as for the $1^{st}$ point carries over by using $\bar{\bar{y}}_{(k,l)} := \frac{\sum_{i=k}^{l} \alpha_{(i)}\bar{y}_{(i)}}{\sum_{i=k}^{l} \alpha_{(i)}}$ and $\bar{R}_{[k,l]} := \sum_{i=k}^{l} \alpha_{(i)}(\bar{y}_{(i)}-\bar{\bar{y}}_{(k,l)})^2$ for all $1 \leq k \leq l \leq d$ as defined in Algorithm~\ref{two passes Dynamic programming algorithm}, and we obtain the result since $Y_r^t$ is constant over every segment.

Since both points hold, then by the definition of $m_{(\hat{d},\hat{\delta})}$ in Algorithm~\ref{two passes Dynamic programming algorithm} we obtain a solution to the minimization criterion in ~\eqref{equation for the algorithm}, thus establishing the correctness of Algorithm~\ref{two passes Dynamic programming algorithm}.

Regarding the complexity of the algorithm, the first step consists in making the $\bar{y}_{[k,l]}$ and $R_{[k,l]}$ matrices. This can be done efficiently by making a cumulative sum matrix $(\sum_{i=k}^{l} y_i)_{k,l}$, whose rows can be formed in $\mathcal{O}(N)$ time and the whole matrix in $\mathcal{O}(N^2)$ time. We compute the $\bar{y}_{[k,l]}$ matrix in $\mathcal{O}(N^2)$ time and $R_{[k,l]}$ in $\mathcal{O}(N^2)$ time, hence this first step has time complexity $\mathcal{O}(N^3)$ and space complexity $\mathcal{O}(N^2)$.

The $1^{st}$ dynamic programming recurrence has time complexity $\mathcal{O}(DN^2)$ and space complexity $\mathcal{O}(DN)$, since there are $\mathcal{O}(DN)$ comparisons to perform in order to find the minimum of $\mathcal{O}(N)$ elements.

$D$ backtracking operations are needed for the $1^{st}$ dynamic programming recurrence. They run in $\mathcal{O}(D)$ time to obtain the optimal models $\tilde{m}_d := \{0=i_0 \leq i_1<i_2<\dots<i_d \leq i_{d+1}=N\}$ from $1$ to $D$. This backtracking procedure has time complexity $\mathcal{O}(D^2)$ and space complexity $\mathcal{O}(D^2)$. 

Each of the $D$ sorting operations that return $\phi_d$ for all $d$ can be done with $\mathcal{O}(D\ln D)$ space and time complexity; more efficient sorting algorithms can be used but since this is not the bottleneck operation in Algorithm~\ref{two passes Dynamic programming algorithm} we do not require more efficiency. On the other hand $\mathcal{O}(D\ln D)$ space for sorting- $D$ sorting stages can be done using the same memory space- is overcome by the $D^2$ storage cost of $D$ storages. More efficient sorting algorithms are described in \citep{CLRS09} and \citep{Goodrich01} Overall, these steps have time complexity $\mathcal{O}(D^2 \ln D)$ and space complexity $\mathcal{O}(D^2)$.

For the second preprocessing steps we need to compute $(\alpha_{(k)})_{k=0}^{d}$, $\bar{\bar{y}}_{(k,l)}$ and $\bar{R}_{[k,l]}$. As before, we do this via a cumulative sum matrix $(\alpha_{(k)})_{k=0}^{d}$ which is built in $\mathcal{O}(D^2)$ time, and a weighed cumulative sum matrix $(\sum_{i=k}^{l} \alpha_{(i)}\bar{y}_i)_{k,l}$ built in $\mathcal{O}(D^2)$ time. We can then compute $(\bar{\bar{y}})_{[k,l]}$ in $\mathcal{O}(D^2)$ and $R_{[k,l]}$ in $\mathcal{O}(D^2)$ time, and doing this for $D$ models requires a time complexity of $\mathcal{O}(D^4)$ and a space complexity of $\mathcal{O}(D^2)$.

The $2^{nd}$ dynamic programming step with backtracking now requires a time complexity of $\mathcal{O}(D^4)$ and a space complexity of $\mathcal{O}(D^3)$ to store $\tilde{\tilde{m}}_{(d,\delta)}$ for all $d$ and $\delta$.

Computing $B_{(d,\delta)}$ requires obtaining $\pen((d,\delta))$ (see \eqref{eq:pen} and \eqref{eq:def_pm}), which can be done recursively using the s~\eqref{binomial coefficient} and \eqref{recursive Stirling} in $\mathcal{O}(DN + D^2)$ in time and $\mathcal{O}(DN + D^2)$ in space; the  minimization to compute $\Crit(m_{(\hat{d},\hat{\delta})})$, which requires $\mathcal{O}(D^2)$ time and $\mathcal{O}(1)$ space; and backtracking to obtain $m_{(\hat{d},\hat{\delta})}$, which requires $\mathcal{O}(D)$ time and $\mathcal{O}(D)$ space, since everything is already stored so we just need to look up $\tilde{m}_{\hat{d}}$ and rearrange back $\tilde{\tilde{m}}_{(\hat{d},\hat{\delta})}$ using $\phi_{\hat{d}}$.

The overall algorithm is dominated by $\mathcal{O}(N^3+D^4)$ time complexty and $\mathcal{O}(N^2+D^3)$ space complexity.
\end{proof}

The time and space complexity can be improved upon using the efficient implementation in \cite{Celisse17} for the dynamic programming of Steps 3 and 11, this implementation change the order of the for loops of $d'$ and $d''$ and computes recursively the values of $ R_{[i,n]}$ and $\bar{R}_{[i,t]}$ as needed, so no preprocessing of steps 1 and 9 is needed. With this we get a time and space complexity $\mathcal{O}(N^2D+D^4)$ and $\mathcal{O}(DN+D^3)$, respectively. The useful regime in which the result of this algorithm are significant is $d_{m^*}'' = o(\frac{N}{\log N}) $ according to corollary \ref{useful regime} so we only to choose $D$ within those constraint. To balance computational performance and statistical performance we can choose $D = \mathcal{O}(N^{1/2})$ giving us a time and space complexity $\mathcal{O}(N^{5/2}$ and $\mathcal{O}(N^{3/2})$, respectively.   

\subsection*{Appendix C: Model Selection criterion for change points and clustering}
\begin{proof}[Proof of \eqref{apposteriori}]
 We start by computing the probability distribution of $F$; the law of total probability yields
\begin{align*}
d\mu_F = \sum_{m \in \mathcal{M}} p_m d\mu_{F/m},
\end{align*}
and Bayes' theorem then provides the posterior distribution of $Y/m$,
\begin{align*}
\frac{d\mu_{m/Y}}{d\mu_m} &= \frac{d\mu_{Y/m}}{d\mu_Y}  \\
&= \frac{d\mu_{Y/m}}{\int d\mu_{Y/m'} d\mu_{m'}} \\
 &=  \frac{\int d\mu_{Y/F}d\mu_{F/m}}{\sum_{m' \in \mathcal{M}} p_{m'}\int d\mu_{Y/F}d\mu_{F/m'}}.
\end{align*}
Both $\mu_{m/Y}$ and $d\mu_m$ are absolutely continuous with respect to the counting measure, hence we have 
\begin{align*}
\frac{d\mu_{m/Y}}{d\mu_m} = \frac{p_{m/Y}}{p_m}.
\end{align*}
We denote by $\phi_N$ the density of the multivariate $\mathcal{N}(0,I_N)$ distribution. The law of total probability again gives
\begin{align*}
d\mu_{Y/m} = \int_{f \in \mathcal{F}_m} \phi_N \left(\frac{Y-f}{\sigma}\right) l_{f/m}(f)df.
\end{align*}

Putting these conditional probabilities together, we obtain the following a-posteriori distribution for the random variable $m$ given the observation $Y$:
\begin{align*} 
p_{m/Y}=\frac{p_m \displaystyle \int_{f \in \mathcal{F}_m} \phi_N \left(\frac{Y-f}{\sigma}\right) l_{f/m}(f)df}{\sum_{m' \in \mathcal{M}} p_{m'} \displaystyle \int_{f' \in \mathcal{F}_m} \phi_N \left(\frac{Y-f'}{\sigma}\right) l_{f/m'}(f')df'}. 
\end{align*}
This complete the proof.
\end{proof}

\begin{lemma} \label{approx lemma}
Let $l\colon\mathbb{R} \to \mathbb{R}^+$ be a four times differentiable probability density function.
Define, for  $f \in \mathbb{R}$, $y^n \in \mathbb{R}^n$:
\begin{align*}
L_n(f,y^n)&:=-\frac{\|y^n-f \mathbbm{1}^n\|_2^2}{2n\sigma^2} + \frac{1}{n}\ln l(f), \\
\sigma^2_{L_n}(f)&:=\frac{1}{|L_n''(f,y^n)|}.
\end{align*}
Let $\mathcal{A}\subset \mathbb{R}^{\infty}$ such that for all $y \in \mathcal{A}$ the following holds:
\begin{itemize}
  \item the integrals $\int_{\mathbb{R}} \exp(L_n(f,y^n))df$ are bounded uniformly (in $n$ and $y^n$) by some constant $\beta$.
  \item the sequence $\hat{f}_n := \frac{\sum_{k=1}^{n}y_k}{n}$ converges and has as limit $\hat{f} = \lim \limits_{n \to \infty} \frac{\sum_{k=1}^{n}y_k}{n}$.
  \item $L_n(\cdot, y^n)$ has a sequence of maximizers $(\bar{f}_n(y^n))_{n=1}^{\infty}$ at which $L_n(\cdot, y^n)$ has a negative second derivative.
  \item there is a $\delta_0,\ N_0$ and $m,M>0$ such that for all $n \geq N_0$ we have $|L_n^{(i)}(f,y^n)|<M$ for all $|f-\bar{f}_n|<\delta_0$ for $i \in \llbracket 2,4 \rrbracket$ and $m<|L_n^{(2)}(f,y^n)|$, where $L_n^{(i)}$ is the $i^{th}$-derivative of $L_n(f,y^n)$ with respect to $f$. 
\end{itemize}
then the sequence $(\bar{f}_n(y^n))_{n \in \mathbb{N}}$ converges to $\bar{f} = \lim \limits_{n \to \infty} \frac{\sum_{k=1}^{n}y_k}{n}$ and we have:
\begin{multline} \label{laplace approx}
\int_{ \mathbb{R}} \exp\left(-\frac{\|y^n-f \mathbbm{1}^n\|_2^2}{2\sigma^2} + \ln l(f)\right) df \\
= \frac{\sqrt{2\pi}\sigma_{L_n}(\bar{f}_n)}{\sqrt{n}} \exp(nL_n(\bar{f}_n)) (1+\mathcal{O}(n^{-3/2})).
\end{multline}
\end{lemma}

\begin{proof}
We begin by observing that $\hat{f}_n = \frac{\sum_{i=1}^n y_i}{n} = \arg\max_{f \in \mathbb{R}} -\frac{\|y^n-f \mathbbm{1}^n\|_2^2}{2n\sigma^2}$. We will show that $(\bar{f}_n)_{n=1}^{\infty}$ should have the same limit as $(\hat{f}_n)_{n=1}^{\infty}$.\\
Since $\frac{d^2}{d f^2} \sum_{k=1}^n (y_k-f )^2 = 2n > 0$, we have, by integration, that
\begin{align*}
- \sum_{k=1}^n (y_k-f )^2 + \sum_{k=1}^n (y_k-\hat{f}_n )^2 
= -n (f-\hat{f}_n )^2.
\end{align*}
Thus, we obtain
\begin{align*}
&\lim \sup \limits_{n \to \infty} \frac{1}{2n\sigma^2} \sup \limits_{|f-\hat{f}_n|>\delta} \left\{- \sum_{k=1}^n (y_k-f )^2 + \sum_{k=1}^n (y_k-\hat{f}_n )^2 \right\} \\
&= \lim \sup \limits_{n \to \infty} \frac{1}{2\sigma^2} \sup \limits_{|f-\hat{f}_n|>\delta} -(f-\hat{f}_n )^2 = -\frac{\delta^2}{2\sigma^2}.
\end{align*}
If there is $\delta_0>0$ and a subsequence $(n_l)_{l=1}^{\infty}$ such that $|\hat{f}_{n_l}-\bar{f}_{n_l}|>\delta_0$ for all $l$, then
\begin{align*}
- \frac{1}{2{n_l}\sigma^2}\sum_{k=1}^{n_l} (y_k-\bar{f}_{n_l} )^2  \leq -\frac{\delta_0^2}{2\sigma^2} - \frac{1}{2{n_l}\sigma^2}\sum_{k=1}^{n_l} (y_k-\hat{f}_{n_l} )^2.
\end{align*}
By assumption $ \frac{\sum_{k=1}^{n}y_k}{n} = \hat{f}_n \to \hat{f} $ then $\ln l(\hat{f}_n) \to \ln l(\hat{f})$. Thus, there is an $N_1$ such for all $n \geq N_1$ we obtain
\begin{align*}
-\frac{1}{n}+\frac{l(\hat{f})}{n} \leq \frac{l(\hat{f}_n)}{n}.
\end{align*}
On the other hand, $l$ is bounded by some $C$, hence
\begin{align*}
\frac{1}{n}\ln l(\bar{f}_n) \leq \frac{1}{n}\ln(C).
\end{align*}
This gives
\begin{align*}
\frac{1}{n}\ln l(\bar{f}_n) - \frac{1}{n}\ln l(\hat{f}_n) \leq \frac{1}{n}\ln(C)+\frac{1}{n}-\frac{1}{n}\ln l(\hat{f}) = \mathcal{O}\left(\frac{1}{n}\right).
\end{align*}
Therefore, for all $\hat{f}_{n_l}$ and $\bar{f}_{n_l}$ such that $|\hat{f}_{n_l}-\bar{f}_{n_l}|>\delta_0$ we have that
\begin{align} \label{eq:ineq_Hn}
L_{n_l}(\bar{f}_{n_l} ,y^{n_l})=&- \frac{1}{2{n_l}\sigma^2}\sum_{k=1}^{n_l} (y_k-\bar{f}_{n_l} )^2 + \frac{1}{{n_l}}\ln l(\bar{f}_{n_l}) \nonumber \\
\leq& - \frac{1}{2{n_l}\sigma^2}\sum_{k=1}^{n_l} (y_k-\hat{f}_{n_l} )^2 + \frac{1}{{n_l}}\ln l(\hat{f}_{n_l}) +\mathcal{O}(n_l^{-1}) -\frac{\delta_0^2}{2\sigma^2} \nonumber \\
=& L_{n_l}(\hat{f}_{n_l},y^{n_l}) -\frac{\delta_0^2}{2\sigma^2} +\mathcal{O}(n_l^{-1}).
\end{align}
This implies that there is an $l_0$ such that $L_{n_{l_0}}(\hat{f}_{n_{l_0}}, y^{n_{l_0}})< L_{n_{l_0}}(\bar{f}_{n_{l_0}},y^{n_{l_0}}) -\frac{\delta_0^2}{4\sigma^2}$, which in turn contradicts the fact that $\bar{f}_n$ is a maximizer of $L_n$ for all $n$.\\
From here we conclude that for all $\delta$ there is $N_2 \in \mathbb{N}$ for all $n\geq N_2$ we have $|\bar{f}_n-\hat{f}_n|<\delta/2$, and since the sequence $(\hat{f}_n)_{n=1}^{\infty}$ converges, there is an $N_3$ for all $n \geq N_3$ we have that $|\hat{f}_n -\hat{f}| <\delta/2$, which in turn implies that for all $n \geq N_4 = \max\{N_2,N_3\}$ we have $|\bar{f}_n-\hat{f}| <\delta$. Thus our original claim and first part of the lemma is proved, namely, that the sequence $(\bar{f}_n)_{n=1}^{\infty}$ is convergent and $\hat{f} = \lim \limits_{n \to \infty} \hat{f}_n = \lim \limits_{n \to \infty} \bar{f}_n = \bar{f}$.\\
For a fixed $\delta > 0$ we will approximate both terms of the following decomposition:
\begin{multline*}
\int_{ \mathbb{R}} \exp(nL_n(f,y^n))df = \int_{ \mathbb{R}-(\hat{f}-\delta,\hat{f}+\delta)} \exp(nL_n(f,y^n))df \\+\int_{ (\hat{f}-\delta,\hat{f}+\delta)} \exp(nL_n(f,y^n))df.
\end{multline*}
Just as before, since $\bar{f}_n \to \bar{f}$ there is an $N_5$ such for all $n \geq N_5$ , for all $f \in \mathbb{R}- (\hat{f}_n-\delta,\hat{f}_n+\delta)$,
\begin{align*}
-\frac{1}{2n\sigma^2} \sum_{k=1}^n (y_k-f )^2 \leq - \frac{1}{2n\sigma^2} \sum_{k=1}^n (y_k-\hat{f}_n )^2 -\frac{\delta^2}{2\sigma^2}.
\end{align*}
Since both $\hat{f}_n$ and $\bar{f}_n$ have the same limit, then by continuity there is $N_6$ starting from which we get:
\begin{align*}
-\frac{1}{2n\sigma^2} \sum_{k=1}^n (y_k-f )^2 \leq - \frac{1}{2n\sigma^2} \sum_{k=1}^n (y_k-\bar{f}_n )^2 -\frac{\delta^2}{4\sigma^2}.
\end{align*}
Following the same steps leading to \eqref{eq:ineq_Hn}, we obtain: 
\begin{align*}
\frac{1}{n}\ln l(f) \leq& \frac{1}{n}\ln l(\bar{f}_n) + \mathcal{O}\left(\frac{1}{n}\right)
\leq \frac{1}{n}\ln(C) + \mathcal{O}\left(\frac{1}{n}\right).
\end{align*}
These two results together imply for that there is an $N_7$ such that for all $n \geq N_7$ we have
\begin{align*}
L_{n}(f ,y^{n}) \leq -\frac{\delta^2}{8\sigma^2}.
\end{align*}
Hence for $n \geq N_6$ we can bound the first integral as follows:
\begin{align*}
&\int_{ \mathbb{R}-(\bar{f}_n-\delta,\bar{f}_n+\delta)} \exp(nL_n(f,y^n))df \\
&= \int_{ \mathbb{R}-(\bar{f}_n-\delta,\bar{f}_n+\delta)} \exp((n-1)L_n(f,y^n)) \exp(L_n(f,y^n)) df \\
&\leq \exp\left(-(n-1)\frac{\delta^2}{8\sigma^2}\right)\int_{ \mathbb{R}} \exp(L_n(f,y^n))df \\
&\leq \exp\left(-(n-1)\frac{\delta^2}{8\sigma^2}\right)\beta.
\end{align*}
Now we turn to the integral over $(\bar{f}_n-\delta,\bar{f}_n+\delta)$. Taking the Taylor series of $L_n$ around $\bar{f}_n$ (and omitting the second argument, $y^n$, for simplicity) we obtain
\begin{multline*}
L_n(f)=L_n(\hat{f})+(f -\bar{f}_n)L_n^{(1)}(\bar{f}_n)+\frac{1}{2}(f -\bar{f}_n)^2 L_n^{(2)}(\bar{f}_n) \\+\frac{1}{6}(f -\bar{f}_n)^3 L_n^{(3)}(\bar{f}_n) + \mathcal{O}\left((f -\bar{f}_n)^4 \right),
\end{multline*}
where $L_n^{(1)}(\bar{f}_n)=0$, and using the Taylor expansion of $\exp(y)$ around $0$ for the higher order terms we obtain:  
\begin{align*}
\exp(nL_n(f)) &= \exp(nL_n(\bar{f}_n))\exp\left(\frac{n}{2}(f -\bar{f}_n)^2 L_n^{(2)}(\bar{f}_n)\right) \\
&\times \left(1+\frac{n}{6}(f -\bar{f}_n)^3 L_n^{(3)}(\bar{f}_n) +n\mathcal{O}\left( (f -\bar{f}_n)^4 \right)\right).
\end{align*}
For the term with odd derivative it is easy to see that the integral is zero
$$\int_{(\bar{f}_n-\delta,\bar{f}_n+\delta)} \frac{n}{6} (f -\bar{f}_n)^3 L_n^{(3)}(\bar{f}_n) \exp\left(\frac{n}{2}(f -\bar{f}_n)^2 L_n^{(2)}(\bar{f}_n) \right) df = 0$$
The big-$\mathcal{O}$ term coming from the residual of the expansion can be neglected since by definition there is a $C \geq 0$ such that
\begin{align*}
&\int_{(\bar{f}_n-\delta,\bar{f}_n+\delta)} n\mathcal{O}\left((f -\bar{f}_n)^4\right) \exp \left( \frac{n}{2}(f -\bar{f}_n)^2 L_n^{(2)}(\bar{f}_n) \right) df \\
&\leq 2Cn\int_{\bar{f}_n}^{\infty} (f -\bar{f}_n)^4 \exp \left( -\frac{n}{2}(f -\bar{f}_n)^2 |L_n^{(2)}(\bar{f}_n)| \right)  df \\
&= Cn\int_{0}^{\infty} u^3 \exp \left(-\frac{n}{2}u^2 |L_n^{(2)}(\bar{f}_n)| \right) d(u^2)\\
&= Cn\int_{0}^{\infty} u^{5/2-1} \exp \left(-\frac{n}{2}u |L_n^{(2)}(\bar{f}_n)| \right) du \\
&= Cn^{-3/2} \frac{\Gamma(5/2)}{|L_n^{(2)}(\bar{f}_n)|} \\
&\leq Cn^{-3/2} \frac{\Gamma(5/2)}{m}\\
&= \mathcal{O}(n^{-3/2}).
\end{align*}
The second derivative term can be approximated as follows:
\begin{align*}
&\int_{(\bar{f}_n-\delta,\bar{f}_n+\delta)}  \exp \left( \frac{n}{2}(f -\bar{f}_n)^2 L_n^{(2)}(\bar{f}_n) \right) df \\
&= \int_{-\infty}^{\infty}  \exp \left( \frac{n}{2}(f -\bar{f}_n)^2 L_n^{(2)}(\bar{f}_n) \right) df + \mathcal{O}\left(e^{-(n-1)\frac{\delta^2}{8\sigma^2}}\right)  \\
&= \frac{\sqrt{2\pi}\sigma_{L_n}(\bar{f}_n)}{\sqrt{n}} + \mathcal{O}\left(e^{-(n-1)\frac{\delta^2}{16\sigma^2}}\right).
\end{align*}
Since the last term is exponentially small on $n$ for fixed $\delta$ putting everything together we get the final claim:
\begin{align*}
\int_{\mathbb{R}} \exp(nL_n(f))df = \frac{\sqrt{2\pi}\sigma_{L_n}(\bar{f}_n)}{\sqrt{n}} \exp(nL_n(\bar{f}_n)) (1+\mathcal{O}(n^{-3/2}))
\end{align*}
\end{proof}
\begin{proof}[Proof of \eqref{upper bound for the MAP}]
Observe first that the conditions of Lemma~\ref{approx lemma} hold for a large class of sufficiently smooth and bounded (possibly improper) priors $l(f)$ and for the data generated according to the sampling scheme in ~\eqref{sampling scheme}, since under these assumptions the  number of samples in each cluster $[k]$ tends to infinity as $N \to \infty$ a.s. and the sample mean converges almost surely (a.s.), \emph{i.e.}, $\hat{f}_k = \lim \limits_{n \to \infty} \hat{f}_{kn}$ a.s. From Lemma~\ref{approx lemma}, the sequence $(\bar{f}_{kn})_{n=1}^{\infty}$ is convergent for every cluster $[k]$, and $\hat{f}_k = \lim \limits_{n \to \infty} \hat{f}_{kn} = \lim \limits_{n \to \infty} \bar{f}_{kn} = \bar{f}_k$ a.s. Thus, by continuity,
\begin{align*}
&\ln \big(\sigma_{L_n}(\bar{f}_{kn})\exp(nL_n(\bar{f}_{kn})) (1+\mathcal{O}(n^{-3/2})) ) \\
&=\ln \big(\sigma_{L_n}(\hat{f}_{kn}) \exp(nL_n(\hat{f}_{kn})) (1+\mathcal{O}(n^{-3/2})) + \mathcal{O}(1)) \big) \\
&=\ln(\sigma_{L_n}(\hat{f}_{kn})) +nL_n(\hat{f}_{kn}) + \ln (1+\mathcal{O}(n^{-3/2})) + \mathcal{O}(1)), 
\end{align*}
with $\sigma^2_{L_n}(f) = \frac{1}{|L_n''(f,y^n)|}= \frac{\sigma^2}{1+\frac{\sigma^2}{n}(\ln(l(f_n)))^{(2)}}$.
From Lemma~\ref{approx lemma} we can rewrite the logarithm of the posterior distribution in \eqref{for approx} as follows: 
\begin{align*}
&\ln p_{m/Y}  \\
&= \ln p_m - \sum_{k=1}^{d'+1} \frac{|[k]|}{2}\ln(2\pi\sigma^2) - \sum_{k=1}^{d'+1} \frac{1}{2} \ln |[k]| +\frac{1}{2}\sum_{k=1}^{d'+1} \ln \sigma_{L_n}^2  \\
&\quad + (d'+1)\ln \sqrt{2\pi} - \sum_{k=1}^{d'+1}  \frac{\|y_{[k]}-\hat{f}_ke^k\|_2^2}{2\sigma^2}  +\sum_{k=1}^{d'+1} \ln\left(1+o\left(\frac{1}{|[k]|}\right)\right) + \mathcal{O}(d'_m)\\
&= \ln p_m - \frac{N}{2}\ln(2\pi\sigma^2) - \sum_{k=1}^{d'+1} \frac{1}{2} \ln |[k]| -\frac{1}{2}\sum_{k=1}^{d'+1} \ln\left(1+\frac{\sigma^2}{|[k]|}\ln l(f_k)\right)^{(2)}  \\
&\quad + (d'+1)\ln(\sqrt{2\pi}\sigma) - \frac{\|y-\hat{f}_m\|_2^2}{2\sigma^2}  +\sum_{k=1}^{d'+1} \ln \left(1 + o \left(\frac{1}{|[k]|}\right)\right)+ \mathcal{O}(d'_m) \\
&= \ln C_N + \ln p_m -\sum_{k=1}^{d'+1} \frac{1}{2} \ln |[k]| -\frac{1}{2}\sum_{k=1}^{d'+1} \ln\left(1+\mathcal{O}\left(\frac{1}{|[k]|}\right)\right)  \\
&\quad + (d'+1)\ln(\sqrt{2\pi}\sigma) -\frac{\|y-\Proj_{\mathcal{F}_m}y\|_2^2}{2\sigma^2}  +\sum_{k=1}^{d'+1} \ln\left(1+o\left(\frac{1}{|[k]|}\right)\right)+ \mathcal{O}(d'_m),
\end{align*}
where $C_N$ is constant depending only on $N$. If we want to maximize the likelihood we would be interested in finding $m \in \mathcal{M}$ such that for all $m' \in \mathcal{M}$ we have $\ln \frac{p_m}{p_{m'}} \geq 0$. Since $\sum_{k=1}^{d'+1} \frac{1}{|[k]|} \leq d'+1$ and $\ln(1+x) = x + o(x)$, this leads to
\begin{align*}
\ln \frac{p_{m/Y}}{p_{m'/Y}} &= \ln \frac{p_m}{p_{m'}} + (d_m'-d_{m'}')\ln(\sqrt{2\pi} \sigma )  \\
&\quad- \frac{1}{2} \left( \sum_{k=1}^{d_m'+1} \ln |[k_m]| - \sum_{k=1}^{d_{m'}'+1} \ln |[k_{m'}]| \right) \\
&\quad - \frac{\|y-\Proj_{\mathcal{F}_m}y\|_2^2-\|y-\Proj_{\mathcal{F}_{m'}}y\|_2^2}{2\sigma^2} +\mathcal{O}(d_m'+d_{m'}'),
\end{align*}
hence maximizing the likelihood is equivalent to minimizing:
\begin{equation*}
\Crit_\MAP(m)=  \ln \frac{1}{p_{m}} +\frac{1}{2}\sum_{k=1}^{d_m'+1} \ln |[k_m]| + \frac{\|y-\Proj_{\mathcal{F}_m}y\|_2^2}{2\sigma^2} +\mathcal{O}(d_m').
\end{equation*}
To avoid the dependency of the criterion on the number of elements in each cluster we observe that $$\sum_{k=1}^{d'+1} \ln |[k]| = (d'+1) \ln(\prod_{k=1}^{d'+1}|[k]|^{\frac{1}{d'+1}}) \leq (d'+1)\ln\frac{N}{d'},$$
from the arithmetic-geometric mean inequality, so we have that $$0 \leq  (d'+1)\ln\frac{N}{d'} -(d'+1) \ln(\prod_{k=1}^{d'+1}|[k]|^{\frac{1}{d'+1}}) \leq  (d'+1)\ln\frac{N}{d'},$$ 
Thus for all $K \geq 1$ we  obtain the following upper bound, which corresponds to \eqref{upper bound for the MAP}:
\begin{equation*} 
\Crit_\MAP(m) \leq \frac{\|y-\Proj_{\mathcal{F}_m}y\|_2^2}{2\sigma^2}+ K\left(\ln \frac{1}{p_{m}} +\frac{1}{2}(d_m'+1)\ln\frac{N}{d_m'}\right)  +\mathcal{O}(d_m').
\end{equation*} 
\end{proof}
\section*{Appendix D: Oracle Inequality and Upper Bound on the Risk}
\begin{lemma} \label{lem:B_N}
There exists a sequence $(B_N)_{N \in \mathbb{N}}$ for which for all $ m \in \mathcal{M}$
\begin{align} \label{eq:def_pm}
p_m=\frac{\exp(-d_m'-d_m'')}{B_N S^2(d_m''+1,d_m'+1)C_{d_m''}^N},
\end{align}
is a valid probability mass function on $\mathcal{M}$, and this sequence satisfies the following bounds:
\begin{align*}
\frac{e^3}{(e-1)^2(e+1)}(1-3e^{-N-1}) \leq B_N \leq \frac{e^3}{(e-1)^2(e+1)}.
\end{align*}
\end{lemma}
\begin{proof}
Fixing $d'$ and $d''$, we can make a model $\pi_m$ by first choosing a subset of cardinality $d''$ from $\{1,2,\dots,N\}$; there are $C_{d''}^N$ ways to make this choice. This leaves us with $d''+1$ segments. We then partition the segments into exactly $d'+1$ parts, which corresponds to taking a partition into $d'+1$ parts of $\{1,2,\dots,d''+1\}$, where the distance between any two elements of the same part is at least two. This can be done in $S^2(d''+1,d'+1)$ ways. Thus if we let $\mathcal{A}(d',d'') = \{\pi_m \in \Pi_N\colon |\pi_m|=d' \textrm{ and }|\pi_m|_0= d''\}$, we have that
\begin{align*}
|\mathcal{A}(d',d'')| = S^2(d''+1,d'+1)C_{d''}^N.
\end{align*}

Since $S^2(\cdot, \cdot)$, $C_{d''_m}^N$ and the exponential are all non-negative, for $p=(p_m)_{m \in \mathcal{M}}$ to be a valid probability mass function we only need to find a positive sequence $(B_N)_{N \in \mathbb{N}}$ such that $p_m$ sum to 1. Now,
\begin{align*}
&\sum \limits_{m \in \mathcal{M}} p_m \\
&= (B_N)^{-1} \sum \limits_{d'=0}^{N-1} \sum \limits_{d''=d'}^{N} \sum \limits_{m \in \mathcal{A}(d',d'')} \left(S^2(d''+1,d'+1)C_{d''}^N \right)^{-1} \exp(-d'-d'') \\
&=(B_N)^{-1}\sum \limits_{d'=0}^{N-1} \sum \limits_{d''=d'}^{N} |\mathcal{A}(d',d'')| \left(S^2(d''+1,d'+1)C_{d''}^N \right)^{-1} \exp(-d'-d'') \\
&= (B_N)^{-1}\sum \limits_{d'=0}^{N-1} \exp(-d')\sum \limits_{d''=d'}^{N}   \exp(-d'') \\
&= (B_N)^{-1}\sum \limits_{d'=0}^{N-1} \exp(-d') \frac{e^{-d'}-e^{-N-1}}{1-e^{-1}} \\
&= \frac{(B_N)^{-1}}{1-e^{-1}} \left(\frac{1-e^{-2N}}{1-e^{-2}} - \frac{e^{-N-1}-e^{-2N-1}}{1-e^{-1}} \right) \\
&= \frac{(B_N)^{-1}e^3}{(e-1)^2(e+1)} \left(1-e^{-2N}-e^{-N-1}+e^{-2N-1}-e^{-N-2}+e^{-2N-2}\right) \\
&= 1,
\end{align*}
which holds for the choice:
$$B_N = \frac{e^3}{(e-1)^2(e+1)} \left(1-e^{-2N}-e^{-N-1}+e^{-2N-1}-e^{-N-2}+e^{-2N-2}\right).$$
The bounds on $B_N$ are easy to verify.
\end{proof}

The choice of probability mass function $(p_m)$ in Lemma~\ref{lem:B_N} distributes the mass evenly among models of the same dimensions. On the other hand, to balance the exponential increase in the number of models the exponential factor makes $p_m$ decrease exponentially with the dimensions. Together with the fact that the prior $l_{Y/m}$ was absorbed in the error term of the approximation in Lemma~\ref{approx lemma}, we obtain a set of what can be considered as least favorable priors for our Bayesian setting, and this will have the effect of reducing the upper bound on the risk. 
\begin{lemma} \label{important bounds}
For $1 \leq k \leq N-1$ and $N \geq 4$, with the convention $0\ln0=0$, the following bounds hold for the binomial coefficients:
\begin{align*}
\left(\frac{N}{k}\right)^k \leq C_{k}^N \leq \left(\frac{Ne}{k}\right)^k,
\end{align*}
and the following bounds hold for the Stirling numbers of the second kind:
\begin{align*}
k^{N-k} \leq S(N,k) \leq \frac{1}{2} (Ne)^k k^{N-2k}.
\end{align*}
In particular, we have that
\begin{multline*}
d_m'' \ln[d_m'' e] - d_m' \ln \frac{d_m''}{e} + d_m'' \ln \frac{N}{d_m''}  \leq \ln \frac{1}{p_{m}} \\
\leq d_m' \ln[d_m''e^2] + d_m'' \ln[d_m'e^2] + d_m'' \ln \frac{N}{d_m''}.
\end{multline*}
\end{lemma}
\begin{proof}
For $1 \leq l < k \leq N$ we have 
\begin{align*}
\frac{N}{k} \leq \frac{N-l}{k-l} \leq \frac{N}{k-l}.
\end{align*}
Hence,
\begin{align*}
\left(\frac{N}{k}\right)^k \leq \frac{N(N-1)\dots(N-l+1)}{k(k-1)\dots1}=C_k^N \leq \frac{N^k}{k!}.
\end{align*}
On the other hand we have that
\begin{align*}
k\ln k-k &=\int_1^k \ln x dx -1 \\
&= \sum_{l=1}^{k-1} \int_l^{l+1} \ln x dx -1 \\
&\leq \sum_{l=1}^{k-1}  \ln(l+1) -1 \\
&\leq \ln(k!),
\end{align*}
thus we obtain $k\ln N - \ln(k!) \leq k (\ln N - \ln k + 1)$, and taking the exponential yields
\begin{align*}
C_k^N \leq \frac{N^k}{k!} \leq \left(\frac{Ne}{k}\right)^k.
\end{align*}
The following bound for Stirling numbers of the second kind can be found in \cite{Rennie69}:
\begin{align*}
\frac{1}{2}(k^2+k+2)k^{N-k-1} - 1 \leq S(N,k) \leq \frac{1}{2}C_{k}^N k^{N-k}, \qquad 1 \leq k \leq N-1.
\end{align*}
Using the upper bound on the binomial coefficients we can easily derive bounds for the Stirling numbers. In particular, since
\begin{align*}
\frac{1}{2}(k^2+k+2)k^{N-k-1} - k^{N-k} &= \frac{1}{2}(k^2+k+2)k^{N-k-1} - \frac{2k}{2}k^{N-k-1} \\
&= \frac{1}{2}(k^2-k+2)k^{N-k-1} \\
&\geq 1,
\end{align*}
we obtain the lower bound
\begin{align*}
S(N,k) \geq \frac{1}{2}(k^2+k+2)k^{N-k-1} - 1 \geq k^{N-k}.
\end{align*}
The upper bound is derived as follows:
\begin{align*}
S(N,k) \leq \frac{1}{2}C_{k}^N k^{N-k} \leq \frac{1}{2} (Ne)^k k^{N-2k}.
\end{align*}
Finally, since $S^2(N+1,k+1)=S(N,k)$, we have from Lemma~\ref{lem:B_N} that
\begin{align*}
 \ln \frac{1}{p_{m}} =& \ln \left( B_N S^2(d_m''+1,d_m'+1) C_{d_m''}^N\exp (d_m' + d_m'')\right) \\
 =& \ln B_N +\ln S(d_m'',d_m') + \ln C_{d_m''}^N + d_m' + d_m'' \\
 \leq& d_m' \ln[d_m''e^2] + d_m'' \ln[d_m'e^2] + d_m'' \ln \frac{N}{d_m''}
\end{align*}
and
\begin{align*}
\ln \frac{1}{p_{m}}  &\geq \ln \frac{e^3}{(e-1)^2 (e+1)} + \ln(1-3e^{-N-1})\\
&\qquad \qquad \qquad+ (d_m''-d_m') \ln d_m''+ d_m' \ln \frac{N}{d_m'} + d_m'+d_m''\\
&\geq d_m'' \ln [d_m''e] - d_m' \ln \frac{d_m''}{e} + d_m'' \ln \frac{N}{d_m''},
\end{align*}
since $\ln \frac{e^3}{(e-1)^2 (e+1)} \approx 0.604$, and for $N \geq 4$ we have $\ln(1-3e^{-N-1}) \geq \ln(1-3e^{-5}) \approx -0.0204$. This leads to the desired result.
\end{proof}
From this lemma, the penalty term in  \eqref{Crit equation} behaves like $d_m''\ln\frac{N}{d_m''}$, which would correspond to the behavior of the first part when the dimensions are close but would penalize more those models with $d'_m \ll d_m''$.
\begin{customthm}{6.1}[Oracle inequality for $\hat{f}_{\hat{m}}$]
With $\mathcal{M}$ restricted to models such that $d'_m \leq N$ and for the choice of $m  \in \mathcal{M}$   corresponding to
\begin{align}
\hat{m} &\in \arg \min \limits_{m \in \mathcal{M}} \|y-\hat{f}_m\|_2^2 + \sigma^2K\pen(m), \\
\pen(m) &:= 2\ln \frac{1}{p_{m}} +(d'+1)\ln\frac{N}{d'} 
\end{align}
with $K=3a$, we obtain for all $a>1$,
\begin{multline*}
E_{f^*}[\|  \Proj \limits_{\mathcal{F}_{\hat{m}} }Y  - f^* \|^2]  \leq \arg\min \limits_{m \in \mathcal{M}} \bigg\{ \frac{a}{a-1} \mathbb{E}_{f^*}[\|  \Proj \limits_{\mathcal{F}_{m} }Y - f^*\|^2] \\
+ \frac{a^2 \sigma^2}{a-1} \left(7 + 3(d_m'+1)\ln\frac{N}{d_m'}+6\ln\frac{1}{p_m} \right) \bigg\}.
\end{multline*}
\end{customthm}
The proof of this theorem follows the line of reasoning described in \cite{Massart03}. To establish this result we rely on the Gaussian concentration inequality stated below and whose proof can be found in \cite{Cirelson76,Boucheron13}:

\begin{lemma}[Tsirelson-Ibragimov-Sudakov Inequality] \label{lem:exp_gauss}
Assume that $F\colon \mathbb{R}^d \to \mathbb{R}$ is $1$-Lipschitz and $X$ is a Gaussian random vector following $\mathcal{N}(0,\sigma^2 I)$. Then, there exists a variable $\xi \sim \textrm{exp}(1)$ following an exponential distribution of parameter 1, such that
\begin{align*} 
F(X) \leq \mathbb{E}[F(X)]+\sigma \sqrt{2\xi}.
\end{align*} 
\end{lemma}

\begin{proof}[Proof of Theorem~\ref{thm:1}]
By definition of $\hat{m}$ we have that, for all $m \in \mathcal{M}$,
\begin{align*}
\|y-\hat{f}_{\hat{m}}\|_2^2 + \sigma^2K\pen(\hat{m}) \leq \|y-\hat{f_m}\|_2^2 + \sigma^2K\pen(m).
\end{align*} 
By expanding the squares and using $Y = f^* + \epsilon$ we obtain
\begin{align*}
\|y-\hat{f}_{m}\|_2^2 = \|f^*-\hat{f}_{m}\|_2^2 + 2\inner{\epsilon}{f^*-\hat{f}_{m}} + \|\epsilon\|_2^2.
\end{align*}
On the other hand, we have by expanding the squares,
\begin{multline}  \label{main equation}
\|f^*-\hat{f}_{\hat{m}}\|_2^2 \leq \|f^*-\hat{f}_{m}\|_2^2 + 2\inner{\epsilon}{f^*-\hat{f}_{m}} \\
- 2\inner{\epsilon}{f^*-\hat{f}_{\hat{m}}} + \sigma^2K\pen(m) -  \sigma^2K\pen(\hat{m}).
\end{multline}
The rest of the proof will consist in upper bounding the expected value of the terms of the right hand side of ~\eqref{main equation}.

Again, since $Y = f^* + \epsilon$ we also have, for all $m \in \mathcal{M}$, that
\begin{align}\label{Projection identity}
\hat{f}_{m} = \Proj \limits_{\mathcal{F}_{m}} Y= \Proj \limits_{\mathcal{F}_{m}} f^* + \Proj \limits_{\mathcal{F}_{m}} \epsilon. 
\end{align}
We can use Lemma~\ref{gaussian projection lemma} to derive a simple bound on $\mathbb{E}[\inner{\epsilon}{f^*-\hat{f}_m}]$ as follows:
\begin{align*}
\mathbb{E}[\inner{\epsilon}{f^*-\hat{f}_{m}}] &= -\mathbb{E}[\inner{\epsilon}{\hat{f}_m}] \\
&= -\mathbb{E}[\inner{\epsilon}{\Proj \limits_{\mathcal{F}_m} f^* + \Proj \limits_{\mathcal{F}_m} \epsilon}] \\
&= -\mathbb{E}[\inner{\epsilon}{\Proj \limits_{\mathcal{F}_m} \epsilon}] \\
&= -\mathbb{E}[\|\Proj \limits_{\mathcal{F}_{m}}\epsilon\|_2^2] \\
&= -\sigma^2 d_m' \leq 0,
\end{align*}
so we can discard this term since it has a negative contribution of small order on the bound.

To bound $2\inner{\epsilon}{\hat{f}_{\hat{m}} - f^*}$ we use Young's inequality $2 \inner{u}{v} \leq a\|u\|_2^2+\frac{1}{a}\|v\|_2^2$ for all $a>0$ as follows:
\begin{align*}
2 \inner{\epsilon}{\hat{f}_{\hat{m}} - f^*} &= 2 \inner{\Proj \limits_{\mathcal{F}_{\hat{m}}\oplus \spanf{f^*}} \epsilon}{\hat{f}_{\hat{m}} - f^*} \\
&= 2 \inner{\Proj \limits_{\mathcal{F}_{\hat{m}}\ominus \spanf{f^*}} \epsilon + \Proj \limits_{\spanf{f^*}} \epsilon}{\hat{f}_{\hat{m}} - f^*} \\
&\leq  a\|\Proj \limits_{\mathcal{F}_{\hat{m}}\ominus \spanf{f^*}} \epsilon + \Proj \limits_{\spanf{f^*}} \epsilon\|_2^2+\frac{1}{a}\|\hat{f}_{\hat{m}} - f^*\|_2^2 \\
&=  a(\|\Proj \limits_{\mathcal{F}_{\hat{m}}\ominus \spanf{f^*}} \epsilon \|_2^2+\|\Proj \limits_{\spanf{f^*}} \epsilon\|_2^2)+\frac{1}{a}\|\hat{f}_{\hat{m}} - f^*\|_2^2, \qquad a > 0.
\end{align*}
This gives
\begin{align} \label{eq3: of main thm}
2 \inner{\epsilon}{\hat{f}_{\hat{m}} - f^*} -\frac{1}{a}\|\hat{f}_{\hat{m}} - f^*\|_2^2 \leq 
a\sigma^2 (\|\Proj \limits_{\mathcal{F}_{\hat{m}}\ominus \spanf{f^*}} \epsilon \|_2^2/\sigma^2+\|\Proj \limits_{\spanf{f^*}} \epsilon\|_2^2/\sigma^2).
\end{align} 
Since $\|\Proj \limits_{\spanf{f^*}} \epsilon\|_2^2/\sigma^2$ follows a $\chi_1^2$ distribution,
\begin{align}\label{eq1: of main thm}
\mathbb{E}[\|\Proj \limits_{\spanf{f^*}} \epsilon\|_2^2/\sigma^2] = 1.
\end{align}
Similarly, for all $m \in \mathcal{M}$, $\|\Proj \limits_{\mathcal{F}_{m} \ominus \spanf{f^*}} \epsilon \|_2^2/\sigma^2$ follows a $\chi_{\bar{d}_m}^2$ distribution, where
\begin{align*}
    \bar{d}_m := \dim (\mathcal{F}_{m} \ominus \spanf{f^*}) =
    \begin{cases*}
      d_m', & if $f^* \in \mathcal{F}_m$, \\
      d_m'+1,        & otherwise.
    \end{cases*}
\end{align*}
Thus,
\begin{align}\label{eq2: of main thm}
\mathbb{E}[\|\Proj \limits_{\mathcal{F}_{m} \ominus \spanf{f^*}} \epsilon \|_2^2/\sigma^2] = \bar{d}_m \leq d_m'+1.
\end{align}
We now use a maximal inequality to bound $\mathbb{E}(a\|\Proj \limits_{\mathcal{F}_{\hat{m}}\ominus \spanf{f^*}} \epsilon \|_2^2+\sigma^2K\pen(\hat{m}))$:
\begin{align} 
&\mathbb{E}\left[\frac{\|\Proj \limits_{\mathcal{F}_{\hat{m}} \ominus \spanf{f^*}} \epsilon \|_2^2}{\sigma^2} - \frac{K}{a} \pen(\hat{m})\right]\\
&\leq \mathbb{E}\left[\max \limits_{m \in \mathcal{M}} \frac{\|\Proj \limits_{\mathcal{F}_{\hat{m}} \ominus \spanf{f^*}} \epsilon \|_2^2}{\sigma^2} - \frac{K}{a} \pen(\hat{m})\right] \nonumber \\
&\leq \sum \limits_{m \in \mathcal{M}} \mathbb{E}\left[\max\left\{ 0, \frac{\|\Proj \limits_{\mathcal{F}_{\hat{m}} \ominus \spanf{f^*}} \epsilon \|_2^2}{\sigma^2} - \frac{K}{a} \pen(\hat{m})\right\}\right]. \label{eq5: of main thm}
\end{align}
On the other hand, since the norm is $1$-Lipschitz, the Gaussian concentration inequality from Lemma~\ref{lem:exp_gauss} implies that there is an exponential random variable $\xi \sim \textrm{exp}(1)$ such that
\begin{align*}
\|\Proj \limits_{\mathcal{F}_{m} \ominus \spanf{f^*}} \epsilon \|_2/\sigma &\leq \mathbb{E}[\|\Proj \limits_{\mathcal{F}_{m} \ominus \spanf{f^*}} \epsilon \|_2/\sigma]+ \sqrt{2\xi} \\
&\leq (\mathbb{E}[(\|\Proj \limits_{\mathcal{F}_{m} \ominus \spanf{f^*}} \epsilon \|_2/\sigma) ^2])^{1/2} + \sqrt{2\xi} \\
&\leq \sqrt{d_m'+1} + \sqrt{2\xi},
\end{align*}
where we used \eqref{eq2: of main thm} in the last step. Taking the square we obtain
\begin{align*}
\frac{\|\Proj \limits_{\mathcal{F}_{m} \ominus \spanf{f^*}} \epsilon \|_2^2}{\sigma^2}
&\leq \left(\sqrt{d_m'+1} + \sqrt{2\xi} \right)^2 \\
&\leq \left(\sqrt{d_m'+1} + \sqrt{2\ln\frac{1}{p_m}} + \sqrt{2 \max\left\{0,\xi-\ln\frac{1}{p_m}\right\}} \right)^2 \\
&\leq \left(\sqrt{(d_m'+1)\ln\frac{N}{d_m'}} + \sqrt{2\ln\frac{1}{p_m}} + \sqrt{2 \max\left\{0,\xi-\ln\frac{1}{p_m}\right\}} \right)^2 \\
&\leq 3(d_m'+1)\ln\frac{N}{d_m'} + 6\ln\frac{1}{p_m} + 6\max\left\{0,\xi-\ln\frac{1}{p_m}\right\} \\
&= 3\pen(m) + 6\max\left\{0,\xi-\ln\frac{1}{p_m}\right\},
\end{align*}
where we used the inequalities $\sqrt{a+b} \leq \sqrt{a} + \sqrt{b}$ in the second step, the assumption $N \geq ed_m'$ in the third step and the inequality $(a+b+c)^2 \leq 3a^2+3b^2+3c^2$ in the fourth step. Since the second term in the last line is nonnegative, this implies also that
\begin{align*}
\max\left\{0, \frac{\|\Proj \limits_{\mathcal{F}_{m} \ominus \spanf{f^*}} \epsilon \|_2^2}{\sigma^2} - 3\pen(m) \right\} \leq 6\max\left\{0,\xi-\ln\frac{1}{p_m}\right\}.
\end{align*}
On the other hand, we have that
\begin{align*}
\mathbb{E}\left[\max\left\{0,\xi-\ln\frac{1}{p_m}\right\}\right] &= \int_{0}^{\infty} \max\left\{0,t - \ln\frac{1}{p_m}\right\} e^{-t}dt \\
&= \int_{\ln\frac{1}{p_m}}^{\infty} \left(t-\ln\frac{1}{p_m}\right) e^{-t} dt \\
&= -\left[\left(t+1-\ln\frac{1}{p_m}\right) e^{-t}\right]_{t = \ln\frac{1}{p_m}}^{\infty}\\
&=p_m,
\end{align*}
hence taking $K=3a$ in ~\eqref{eq5: of main thm} yields
\begin{align*}
&\mathbb{E}\left[\frac{\|\Proj \limits_{\mathcal{F}_{\hat{m}} \ominus \spanf{f^*}} \epsilon \|_2^2}{\sigma^2} - \frac{K}{a} \pen(\hat{m})\right] \\
&\leq \sum \limits_{m \in \mathcal{M}} \mathbb{E}\left[\max\left\{ 0, \frac{\|\Proj \limits_{\mathcal{F}_{\hat{m}} \ominus \spanf{f^*}} \epsilon \|_2^2}{\sigma^2} - 3 \pen(\hat{m})\right\} \right] \\
&\leq 6\sum \limits_{m \in \mathcal{M}} \mathbb{E}\left[\max\left\{0,\xi-\ln\frac{1}{p_m}\right\}\right] \\
&\leq 6\sum \limits_{m \in \mathcal{M}} p_m \\
&\leq 6.
\end{align*}
Combining the last result and equations \eqref{eq3: of main thm}, \eqref{eq1: of main thm} and \eqref{eq5: of main thm} brings us to
\begin{align*}
&\mathbb{E}\left[2 \inner{\epsilon}{\hat{f}_{\hat{m}} - f^*} -\frac{1}{a}\|\hat{f}_{\hat{m}} - f^*\|_2^2 - 3a\sigma^2 \pen(\hat{m})\right] \\
&\leq a\sigma^2\mathbb{E}\left[\frac{\|\Proj \limits_{\mathcal{F}_{\hat{m}}\ominus \spanf{f^*}} \epsilon \|_2^2}{\sigma^2} + \frac{\|\Proj \limits_{\spanf{f^*}} \epsilon\|_2^2}{\sigma^2} - 3 \pen(\hat{m}) \right] \\
&\leq a\sigma^2 \sum \limits_{m \in \mathcal{M}} \mathbb{E}\left[\max\left\{ 0, \frac{\|\Proj \limits_{\mathcal{F}_{m} \ominus \spanf{f^*}} \epsilon \|_2^2}{\sigma^2} - 3 \pen(m)\right\} \right] + a\sigma^2 \\
&\leq 7a\sigma^2.
\end{align*}
Going back to \eqref{main equation}, substituting $K$ by its value and subtracting $\frac{1}{a} \|\hat{f}_{\hat{m}} - f^*\|_2^2$ from both sides we obtain
\begin{multline*}
\frac{a-1}{a} \| \Proj \limits_{\mathcal{F}_{\hat{m}} }Y - f^*\|_2^2 \leq
\|f^*- \Proj \limits_{\mathcal{F}_m }Y \|_2^2 - 2\inner{\epsilon}{ \Proj \limits_{\mathcal{F}_{m} }Y  - f^*} \\
+ 2 \inner{\epsilon}{ \Proj \limits_{\mathcal{F}_{\hat{m}} }Y  - f^*} -\frac{1}{a}\| \Proj \limits_{\mathcal{F}_{\hat{m}} }Y  - f^*\|_2^2 - 3a\sigma^2 \pen(\hat{m}) +3a\sigma^2 \pen(m).
\end{multline*}
Taking the minimum of this expression over $m \in \mathcal{M}$ and the expectation, and omitting negative terms we obtain the desired result for all $a>1$:
\begin{multline*} 
E_{f^*}[\|  \Proj \limits_{\mathcal{F}_{\hat{m}} }Y  - f^* \|^2]  \leq \arg\min \limits_{m \in \mathcal{M}} \{ \frac{a}{a-1} \mathbb{E}_{f^*}[\|  \Proj \limits_{\mathcal{F}_{m} }Y - f^*\|^2] \\
+ \frac{a^2 \sigma^2}{a-1} (7 + 3(d_m'+1)\ln\frac{N}{d_m'}+6\ln\frac{1}{p_m} ) \}.
\end{multline*}
\end{proof}

\begin{customcor}{6.1}
For the set of models described in ~\eqref{regression} with $f^* \in \mathcal{F}_{m^*}$ the following properties hold:

\begin{itemize}
\item Adaptation and Risk Upper bound: The following adaptive upper bound in terms of $d_{m^*}'$ and $d_{m^*}''$ holds for $a=2$:
\begin{align*}
E_{f^*}[\|  \Proj \limits_{\mathcal{F}_{\hat{m}} }Y  - f^* \|^2] &\leq 4 \sigma^2 \bigg( 7 + 3(d_{m^*}'+1)\ln\frac{N}{d_{m^*}'}
\\
&+6\bigg(d_{m^*}' \ln[d_{m^*}''e^{\frac{13}{6}}] + d_{m^*}'' \ln[d_{m^*}'e^2] + d_{m^*}'' \ln \frac{N}{d_{m^*}''} \bigg) \bigg).
\end{align*}

\item Consistency: If $d_{m^*}'' = o(N / \ln N)$, then $\lim_{N \to \infty } N^{-1} \mathbb{E}_{f^*}[\| \hat{f}_{\hat{m}}- f^*\|^2] = 0$.
\end{itemize}
\end{customcor}
\begin{proof}
Equation~\eqref{oracle inequality} implies in particular that for $m^*$ such that $f^* \in \mathcal{F}_{m^*}$ we obtain  
\begin{multline*} 
E_{f^*}[\|  \Proj \limits_{\mathcal{F}_{\hat{m}} }Y  - f^* \|^2]  \leq  \frac{a}{a-1} \mathbb{E}_{f^*}[\|  \Proj \limits_{\mathcal{F}_{m^*} }Y - f^*\|^2] \\
+ \frac{a^2 \sigma^2}{a-1} \left(7 + 3(d_{m^*}'+1)\ln\frac{N}{d_{m^*}'}+6\ln\frac{1}{p_{m^*}} \right) .
\end{multline*}
To simplify the first expectation of the right hand side,
we can use ~\eqref{Projection identity} and Lemma~\ref{gaussian projection lemma} to obtain
\begin{align*}
\mathbb{E}_{f^*}[\|  \Proj \limits_{\mathcal{F}_{m^*} }Y - f^*\|^2] &= \mathbb{E}_{f^*}[\|  \Proj \limits_{\mathcal{F}_{m^*} }\epsilon \|^2]
= \sigma^2d_{m^*}'.
\end{align*}
Then, using Lemma~\ref{important bounds} we can upper bound the rest of the right hand side, and choosing $a= \frac{C}{C-1}$ yieds, for all $C>1$,
\begin{align*}
&E_{f^*}[\|  \Proj \limits_{\mathcal{F}_{\hat{m}} }Y  - f^* \|^2] \\
&\leq C\sigma^2 d_{m^*}' + \frac{C^2 \sigma^2}{C-1} \big(7 + 3(d_{m^*}'+1)\ln\frac{N}{d_{m^*}'}+ \\
&\qquad \qquad \qquad 6\big(d_{m^*}' \ln[d_{m^*}''e^2] + d_{m^*}'' \ln[d_{m^*}'e^2] + d_{m^*}'' \ln \frac{N}{d_{m^*}''} \big) \big) \\
&= \frac{C^2 \sigma^2}{C-1}\big( \big(1-\frac{1}{C}\big) d_{m^*}' +  7 + 3(d_{m^*}'+1)\ln\frac{N}{d_{m^*}'} \\
&\qquad \qquad \qquad +6\big(d_{m^*}' \ln[d_{m^*}''e^2] + d_{m^*}'' \ln[d_{m^*}'e^2] + d_{m^*}'' \ln \frac{N}{d_{m^*}''} \big) \big) \\
&\leq \frac{C^2 \sigma^2}{C-1}\big( 7 + 3(d_{m^*}'+1)\ln\frac{N}{d_{m^*}'} \\
&\qquad \qquad \qquad +6\big(d_{m^*}' \ln[d_{m^*}''e^{\frac{13}{6}}] + d_{m^*}'' \ln[d_{m^*}'e^2] + d_{m^*}'' \ln \frac{N}{d_{m^*}''} \big) \big).
\end{align*}
This upper bound achieves a minimum for $C = 2$, yielding the Adaptation and Risk Upper bound result:
\begin{multline*}
E_{f^*}[\|  \Proj \limits_{\mathcal{F}_{\hat{m}} }Y  - f^* \|^2] \leq 4 \sigma^2 \big( 7 + 3(d_{m^*}'+1)\ln\frac{N}{d_{m^*}'} \\
+6\big(d_{m^*}' \ln[d_{m^*}''e^{\frac{13}{6}}] + d_{m^*}'' \ln[d_{m^*}'e^2] + d_{m^*}'' \ln \frac{N}{d_{m^*}''} \big) \big).
\end{multline*}
Given that $d_{m^*}' \leq d_{m^*}'' \leq N$, this last equation also implies that
\begin{equation*}
E_{f^*}[\|  \Proj \limits_{\mathcal{F}_{\hat{m}} }Y  - f^* \|^2] = \mathcal{O}(d_{m^*}'\ln N + d_{m^*}''\ln N) = \mathcal{O}( d_{m^*}''\ln N).
\end{equation*} 
Hence, as long as $d_{m^*}'' = o(N/\ln N)$, we obtain $\lim_{N \to \infty } N^{-1} \mathbb{E}_{f^*}[\| \hat{f}_{\hat{m}}- f^*\|^2] = 0$, which establishes the Consistency result.
\end{proof}
\end{document}